\documentclass{article}
\usepackage{microtype,nicefrac}
\usepackage{graphicx}
\usepackage{booktabs} 

\usepackage{hyperref}

\usepackage[accepted]{icml2023}

\usepackage{amsmath}
\usepackage{amssymb}
\usepackage{mathtools}
\usepackage{amsthm}

\usepackage[capitalize,noabbrev]{cleveref}

\theoremstyle{plain}
\newtheorem{theorem}{Theorem}[section]
\newtheorem{proposition}[theorem]{Proposition}
\newtheorem{lemma}[theorem]{Lemma}

\theoremstyle{definition}
\newtheorem{definition}[theorem]{Definition}
\newtheorem{assumption}[theorem]{Assumption}
\theoremstyle{remark}
\newtheorem{remark}[theorem]{Remark}

\usepackage{tri_preamble,prof_preamble}
\icmltitlerunning{Deep Clustering With Incomplete Noisy Pairwise Annotations}

\begin{document}

\twocolumn[
\icmltitle{Deep Clustering with Incomplete Noisy Pairwise Annotations: A Geometric Regularization Approach}




\begin{icmlauthorlist}
\icmlauthor{Tri Nguyen}{yyy}
\icmlauthor{Shahana Ibrahim}{yyy}
\icmlauthor{Xiao Fu}{yyy}
\end{icmlauthorlist}

\icmlaffiliation{yyy}{School of Electrical Engineering and Computer Science, Oregon State University, OR, USA}

\icmlcorrespondingauthor{Xiao Fu}{xiao.fu@oregonstate.edu}

\icmlkeywords{constrained clustering, matrix completion, identifiability}

\vskip 0.3in
]



\printAffiliationsAndNotice{}  

\begin{abstract}
The recent integration of deep learning and pairwise similarity annotation-based constrained clustering---i.e., {\it deep constrained clustering} (DCC)---has proven effective for incorporating weak supervision into massive data clustering: Less than 1\% of pair similarity annotations can often substantially enhance the clustering accuracy.
However, beyond empirical successes, there is a lack of understanding of DCC. In addition, many DCC paradigms are sensitive to annotation noise, but performance-guaranteed noisy DCC methods have been largely elusive.
This work first takes a deep look into a recently emerged logistic loss function of DCC, and characterizes its theoretical properties. Our result shows that the logistic DCC loss ensures the identifiability of data membership under reasonable conditions, which may shed light on its effectiveness in practice.
Building upon this understanding, a new loss function based on geometric factor analysis is proposed to fend against noisy annotations. It is shown that even under {\it unknown} annotation confusions, the data membership can still be {\it provably} identified under our proposed learning criterion. The proposed approach is tested over multiple datasets to validate our claims.
\end{abstract}

\section{Introduction}
\label{submission}
Clustering is one of the most prominent unsupervised learning tasks \cite{jain1988algorithms}.
Classic clustering paradigms, e.g., K-means and spectral clustering, are designed to work without any label information. In practice, it was observed that limited supervision may significantly boost the clustering performance. The so-called {\it constrained clustering} (CC) method  \cite{wagstaff2000clustering} is one of such weak (or semi-)supervised approaches. The CC paradigm has many variants (see  \cite{basu2008constrained,schultz2003learning,yeung1996time,davidson2010sat,zhang2019framework,zhang2021framework}). The arguably most widely used paradigm annotates the similarity (using binary codes) of data pairs and restrains the clustering outcomes to respect these pairwise constraints; see, e.g., \cite{basu2004active,segal2003discovering,wagstaff2001constrained,givoni2009semi,hsu2015neural,manduchi2021deep}.  
Note that annotating similarity is considerably easier than annotating the exact class labels of data, yet it was observed that such ``simple'' annotations could boost the performance of data categorization. These findings may assist designing economical annotation mechanisms in the era of data explosion.

Early CC approaches are often combined with existing clustering modules like K-means  \cite{basu2004active,wagstaff2001constrained,bilenko2004integrating,li2009constrained,kamvar2003spectral,wang2014constrained,cucuringu2016simple}. The pairwise annotations are used to construct regularization terms that are similar to graph-Laplacian based regularization. Variants using Bayesian perspectives were also proposed  \cite{basu2004probabilistic,law2005model}. These methods were observed to largely outperform their unsupervised counterparts, e.g., K-means and spectral clustering, even only as few as 1\% of the pairs are annotated.

In recent years, clustering modules and deep neural network-based feature extractors are proposed to be learned jointly---leading to the so-called end-to-end deep clustering approaches; see, e.g.,  \cite{yang2017towards,caron2018deep}. These methods use clustering structures in the latent space to regularize the neural networks, and the neural networks offers enhanced transformation power to find latent spaces where the data can be well clustered.
The CC methods also benefited from similar ideas. The so-called {\it deep constrained clustering} (DCC) approaches were proposed in  \cite{zhang2019framework,zhang2021framework,manduchi2021deep}, and substantial performance enhancement relative to classic CC methods was observed.

Although the DCC methods have enjoyed empirical successes, some notable challenges remain.
First, there is a general lack of understanding to the effectiveness of DCC.
It is unclear under what conditions the loss functions constructed for DCC could succeed or fail in finding the ground-truth cluster membership of the data entities.
However, understanding the {\it identifiability} of the membership is critical for designing principled and robust DCC systems. The interplay of key aspects, e.g., neural network complexity and the generalization ability of the learned feature extractor, in the context of DCC is also of great interest---yet no pertinent study exists.
Second, most of the existing (D)CC methods (implicitly) assumed that annotations are accurate; see, e.g., \cite{basu2004active,wagstaff2001constrained,li2009constrained,zhang2019framework,ren2019semi,hsu2018probabilistic}. Hence, many of these methods may not be robust to annotation noise. This is particularly detrimental to DCC methods as large over-parameterized neural models easily overfit \cite{du2018gradient}.
Some works took noisy annotations into consideration (e.g., \cite{luo2018semi,manduchi2021deep}), but no guarantees of recovering the cluster membership.

\paragraph{Contributions.} In this work, we take a deeper look at the DCC problem from a membership-identifiability analysis viewpoint. Our contribution is twofold:

First, we re-examine a recently emerged effective loss function of DCC, namely, the logistic loss-based DCC criterion, which was proposed by \cite{hsu2018probabilistic,zhang2021framework}. 
We show that, if the pairwise annotations are generated following a model that is reminiscent of the mixed-membership stochastic blockmodel (MMSB) \cite{airoldi2008mixed,huang2019detecting}, then, the logistic loss can provably recover (i) the data entities' cluster membership and (ii) the nonlinear function that maps the data features to the membership indicator for both seen and unseen data---and thus generalization of the learned neural network is guaranteed.

Second, using our understanding, we propose a noisy annotation-robust version of logistic loss for DCC. We explicitly model the annotators' confusion as a probability transition matrix, which is inspired by classic noisy label analysis such as the Dawid-Skene model \cite{dawid1979maximum,ghosh2011moderates,zhang2014spectral}. We propose a geometric factor analysis \cite{fu2018identifiability,fu2019nonnegative} based learning criterion to {\it provably} ensure the identifiability of the ground-truth cluster membership, in the presence of annotation confusion.  

We test our method over a series of DCC tasks and observe that the proposed approach significantly improves the performance over existing paradigms, especially when annotation noise exists. Our finding shows the significance of identifiability in DCC, echoing observations made in similar semi-supervised/unsupervised problems, e.g., \cite{arora2013practical,kumar2013fast,anandkumar2014tensor,zhang2014spectral}. We also evaluate the algorithms using real data collected through the Amazon Mechanical Turk (AMT) platform. The code is published at \url{github.com/ductri/VolMaxDCC}.

\paragraph{Notation.} We use $x$, $\bm{x}$ and $\bm{X}$ to denote scalar, vector and matrix, respectively; 
    both $[\bm{x}]_k, x_k$ refer to the $k$th element of vector $\bm{x}$;
    $X_{ij}$ (and $[\X]_{i,j}$) is the element in the $i$th row and $j$th column of $\bm{X}$;
$\bm{I}_K$ denotes the identity matrix of size $K$; 
$\bm 0$ and $\bm 1$ are all-zero and all-one matrices with proper sizes;
$\langle \bm{x}, \bm{y} \rangle$ and $\langle \bm{X}, \bm{Y} \rangle$ denote dot products between two vectors and two matrices, respectively; $\norm{\bm{x}}$ denotes the $\ell_2$-norm; $\norm{\bm{X}}_{\rm F}$ and $ \norm{\bm{X}}_2$ denote the Frobenius norm and spectral norm of $\bm{X}$, respectively; 
$\sigma_{\max}(\bm{X}), \sigma_{\min}(\bm{X})$, and $ \sigma_i(\bm{X})$ represent the largest, the smallest, and the $i$th singular value of matrix $\bm{X}$, respectively;
$[N]$ is the set of natural numbers from 1 to $N$, i.e., $[N]=\{1,\ldots,N\}$; $[N]\times [N]$ denotes the set of all possible pairs of $(i,j)$ where $i \in [N], j\in [N]$;
${\rm cone}(\bm{X})$ to denote conic hull of the column vectors of $\bm{X}$, i.e., ${\rm cone}(\bm{X}) = \set{\bm{y} \mid \bm{y} = \bm{X} \boldsymbol \theta , \boldsymbol \theta \geq \bm 0}$.

\section{Background}
\paragraph{Problem Setting.}
We consider the CC setting as follows:
There are $N$ data samples $\bm{x}_1, \ldots , \bm{x}_N$. Each sample belongs to one or multiple clusters of in total $K$ clusters.
The association of $\x_n$ with the clusters is represented by a vector $\bm m_n\in\mathbb{R}^{K}$. The element $[\bm m_n]_k$ represents the probability that data $n$ belongs to cluster $k$.
Note that in hard clustering, $[\bm m_n]_k\in\{0,1\}$ for all $k\in [K]$, which means the clusters have no overlaps. In more general cases, we have 
\begin{align}\label{eq:msimplex}
    \bm 1^\T \bm m_n=1,~\bm m_n\geq \bm 0.
\end{align}
A collection of $M$ pairwise annotations are available, which are denoted by $(i_1, j_1, y_1), \ldots , (i_M, j_M, y_M)$. Here, 
\begin{align}
    y_m = \begin{cases}
        1,~&\text{$\bm{x}_{i_m}, \bm{x}_{j_m}$ are ``similar''},\\
        0,~&\text{otherwise},
    \end{cases}
\end{align}
where the similarity of the membership of $\x_{i_m}$ and $\x_{j_m}$ is often deemed by an annotator.
Note that there are in total $N(N-1)/2$ such data pairs, and we often have
\[  M\ll N(N-1)/2; \]
that is, only a small portion of the data pairs are annotated.
The objective of pairwise annotation-based CC is to find the cluster membership vector $\bm m_n$ of each $\x_n$ using the data and the $M$ annotations.

\paragraph{Early CC Methods. }
The task of clustering with the pairwise constraints can be dated back to the early 2000s, where \cite{wagstaff2000clustering,wagstaff2001constrained} used the pairwise annotations to impose extra constraints of the classic K-means iterations.
The work \cite{basu2004active} considered using pairwise annotation-induced ``soft'' constraints (or, regularization terms) to modify K-means; see similar ideas in \cite{bilenko2004integrating}. Instead of modifying K-means, another line of approaches proposed to work with pairwise constraints under the spectral clustering framework.
The idea is to incorporate the pairwise annotation-based constraints into the construction of the graph affinity matrix; see, e.g., \cite{kulis2005semi,lu2008constrained,li2009constrained,cucuringu2016simple}. 

\paragraph{DCC Developments. }
In recent years, deep neural network-based feature extractors were proposed to combine with CC---leading to the {\it deep constrained clustering} (DCC) paradigms, e.g., \cite{manduchi2021deep,luo2018semi,zhang2019framework,zhang2021framework}.
In DCC, a deep neural network (DNN) $\bm f_{\bm \theta}(\cdot)$ is used to link the data vector $\x_{n}$ with its membership, i.e.,
\[ \bm m_{n} = \bm f_{\bm \theta}(\x_{ n}). \]
The use of DNN helps nonlinearly transform the data to spaces that are ``friendly'' to clustering \cite{yang2017towards}.
Learning $\bm f_{\bm \theta}(\cdot)$ also allows the neural network to generalize to unseen data.
By \eqref{eq:msimplex}, $$\bm m_{i_m}^\T \bm m_{j_m}\in [0,1]$$ can be considered as the probability that $\x_{i_m}$ and $\x_{j_m}$ belong to the same cluster, i.e., $y_m\sim {\sf Bernoulli}(\bm m_{i_m}^\T \bm m_{j_m})$.
From this perspective, recent works have proposed an extension of logistic regression to incorporate pairwise annotations \cite{hsu2018probabilistic, zhang2019framework,zhang2021framework}: 
\begin{align}\label{eq:logistic_cc}
{\sf Loss}_{\rm cc}(\bm \theta)=\dfrac{1}{M}\sum^{M}_{m=1} \Big( y_m \log \dfrac{1}{\bm{f}_{\boldsymbol \theta}(\bm{x}_{i_m})^{\T} \bm{f}_{\boldsymbol \theta}(\bm{x}_{j_m})} + \\
(1-y_m) \log \dfrac{1}{1- \bm{f}_{\boldsymbol \theta}(\bm{x}_{i_m})^{\T} \bm{f}_{\boldsymbol \theta}(\bm{x}_{j_m})} \Big). \nonumber
\end{align} 
In the literature, the ${\sf Loss}_{\rm cc}$ term is sometimes used with other loss functions; e.g., \cite{zhang2019framework,zhang2021framework} used an overall loss function consisting of two terms:
\begin{align}
   \minimize~ {\sf Loss}_{\rm recon} + \lambda {\sf Loss}_{\rm cc},
\end{align}
where $\lambda\geq 0$ and the reconstruction loss $ {\sf Loss}_{\rm recon}$ was realized using an autoencoder. Such combination is advocated for practical reasons, e.g., utilizing all available data. Nonetheless, as we will show, ${\sf Loss}_{\rm cc}$ itself suffices to offer strong guarantees under reasonable conditions.

\paragraph{Challenges.} Although there have been abundant empirical evidence demonstrating the effectiveness of CC and DCC, theoretical understanding has been largely behind. This is particularly obvious for the DCC case, where aspects such as the identifiability of $\bm m_n$ and the generalization ability of the learned $\bm f_{\bm \theta}(\cdot)$ are of great interest---yet no theoretical support exists, to our best knowledge.

Another challenge lies in noise robustness. Although it has been widely observed that noisy annotations could greatly impact the performance of CC and DCC \cite{liu2017partition,covoes2013study,pelleg2007k,manduchi2021deep,luo2018semi,chang2017multiple,zhang2021framework,zhu2015constrained}, effective solutions---especially performance-guaranteed ones---for handling this problem have been largely lacking.
Many works did test their algorithms with noisy labels (see, e.g., \cite{cucuringu2016simple,wang2014constrained}), but no special care was taken to alleviate the impact of such noisy labels.
Some heuristics---such as pre-processing the data \cite{yi2012semi}, modeling annotation uncertainties \cite{manduchi2021deep}, and introducing concepts reflecting human behaviors \cite{luo2018semi,chang2017multiple} and annotators' accuracy \cite{luo2018semi,chang2017multiple}---were also proposed in the literature.
However, performance guarantees have been elusive.

\section{DCC Loss Revisited: Identifiability and Generalization}
In this section, we take a deeper look into the DCC loss in \eqref{eq:logistic_cc} 
and understand its theoretical properties. Such understanding will allow us to design a performance-guaranteed new DCC loss in the presence of noisy pairwise annotations.

\subsection{A Generative Model of Annotations}
\label{sub:generative_model}
To better present the results, in this section, we use the superscript ``$\natural$'' to denote all the ground-truth terms; e.g., $\bm f^\natural(\cdot)$ denotes the ground-truth nonlinear mapping from data to membership and $\bm m_n^\natural=\bm f^\natural(\bm x_n)$ denotes the ground-truth membership vector of sample $n$.
We propose to employ the following generative model of $y_m$:  Given $\bm{x}_1, \ldots , \bm{x}_N \sim \mathcal{P}_{\mathcal{X}}, \bm{x}_n \in \mathcal{X}$, 
\begin{subequations}
\label{eq:generative_model}
\begin{align}
   &\text{$i, j$ are sampled over $[N] \times [N]$ };\\
   &\text{$\m_i^\natural=\bm f^\natural(\x_i)$ and $\m_j^\natural=\bm f^\natural(\x_j)$};\\
   &\text{$y_{i,j} \sim {\sf Bernoulli}( {\langle}  \bm m_i^\natural,\bm m_j^\natural {\rangle} ) $}.
\end{align}
\end{subequations}

Note that the logistic loss  in \eqref{eq:logistic_cc} is the maximum likelihood estimator (MLE) of the parameters in the generative model. The model is reminiscent of the classic generative models of logistic regression and network analysis, particularly, MMSB \cite{airoldi2008mixed}. 
In MMSB, the nonlinear mapping from the data features to the membership vectors were not considered. 
Incorporating the nonlinear mapping $\bm f^\natural$ follows the ideas from supervised learning, where the relationship between the data features and the data class is often modeled as the following conditional probability represented by $\bm f^\natural$ \cite{shalev2014understanding}:
\[ [\bm m^\natural_n]_k  = {\sf Pr}(y=k|\x_n) = [\bm f^\natural(\x_n)]_k.\]
Once $\bm f^\natural$ is learned, it can be used as a multi-class classifier.
This perspective was also mentioned in \cite{hsu2018probabilistic,zhang2019framework,zhang2021framework}---for algorithm design purpose. However, membership identifiability was not addressed.

In this section, we will show that the logistic loss \eqref{eq:logistic_cc} ensures identifying the membership vectors $\bm m_n$ under the generative model in \eqref{eq:generative_model}.
It also ensures that $\bm f_{\bm \theta} \approx \bm f^\natural$, under reasonable conditions. These findings may shed some light onto the effectiveness and good generalization performance of DCC using \eqref{eq:logistic_cc} in the literature.

\subsection{Performance Analysis}

\paragraph{Finite-Sample Identifiability and Generalization.} We first show that ${\sf Loss}_{cc}$ is a sound criterion for identifying $\bm m_n^\natural$ and $\bm f^\natural(\cdot)$ by itself. Specifically, let us denote
\begin{equation}
    \label{eq:theta_star}
    \bm \theta^{\star } =\arg\min_{\bm \theta} {\sf Loss}_{\rm cc}(\bm \theta)
\end{equation}
and ${\bm f}^{\star }=\bm f_{{\bm \theta}^{\star }}$.
Here, $\bm{f}_{\boldsymbol \theta}$ is represented by a deep neural network. To be more specific, we consider $\bm{f}_{\boldsymbol \theta}$ belonging to a function class  $\mathcal{F}$ defined by
\[
\mathcal{F} \triangleq \set{ \texttt{softmax}(\texttt{net}(\bm{x}; \boldsymbol \theta))~|~\forall \x\in {\cal X}},
\]
where $\texttt{net}(\cdot;\bm \theta)$ is a neural network that maps $\bm{x}$ to $\mathbb{R}^{K}$, and $[\texttt{softmax}(\bm{x})]_k \triangleq \exp (x_k) / \sum^{K}_{\ell =1} \exp (x_\ell)$ is imposed onto the output layer of the neural network, which is used to reflect the constraints in \eqref{eq:msimplex}.

We will show that ${\bm f}^{\star }\approx \bm f^\natural$ and ${{\bm m}^{\star }_n={\bm f}^{\star }(\x_n) \approx \bm m_n^\natural}$ under reasonable conditions. To proceed, let us invoke the following assumptions:

\begin{assumption}[Anchor Sample Condition (ASC)]
\label{assumption:asc}
Let $\bm{m}_n^{\natural} = \bm{f}^{\natural}(\bm{x}_n)$, $\bm{M}^{\natural} = [\bm{m}^{\natural}_1, \ldots , \bm{m}^{\natural}_N]$. $\bm{M}^{\natural}$ satisfies ASC if there exists a set $\mathcal{K}$ of $K$ indices such that $\bm{M}^{\natural}[:, \mathcal{K}] = \bm{I}$.
Accordingly, define $\bm{V}^{\natural} \triangleq \bm{M}^{\natural}[:, \mathcal{K}^{c}], \mathcal{K}^{c} \triangleq [N]\setminus \mathcal{K}$.
\end{assumption}
\begin{assumption}(Function Class)
\label{assumption:F_cond1}
There exist $0 < \nu < 1$ and $\widetilde{\bm{f}} \in \mathcal{F}$ such that
\begin{align*}
\norm{\widetilde{\bm{f}}(\bm{x}) - \bm{f}^{\natural}(\bm{x})}  \leq \nu, \forall \bm{x} \in \mathcal{X}. \numberthis\label{assumption:prox}
\end{align*}
In addition, $\alpha<\bm{f}(\bm{x})^{\T} \bm{f}(\bm{y})<1-\alpha, \; \forall \bm{x}, \bm{y} \in \mathcal{X}, \forall \bm{f} \in \mathcal{F}$, for some $0< \alpha<1 $. The complexity measure of the neural network $\texttt{net}(\bm{x}; \boldsymbol \theta)$ is $R_{\texttt{NET}}$.
\end{assumption}
The ASC means that there exist samples that solely belong to a single cluster, which are called the anchor samples. This assumption is reminiscent of the anchor point assumption in the community detection literature \cite{panov2017consistent,mao2017mixed}.
Condition \eqref{assumption:prox} takes the approximation error of the employed neural network class ${\cal F}$ into consideration. 
The assumption on $\alpha$ is a regularity condition that prevents pathological unbounded cases of the logistic loss from happening.
The constant $R_{\texttt{NET}}$ is proportional to the upper bound of the so-called spectral complexity in \cite{bartlett2017spectrally}. A formal definition of $R_{\texttt{NET}}$ is given in Lemma~\ref{lemma:covering_number_resnet}. 
The parameters $\nu$ and $R_{\texttt{NET}}$ present a tradeoff: Roughly speaking, if one has a deeper and wider neural network, then $\nu$ is smaller---but $R_{\texttt{NET}}$ is bigger.

Define $\bm{P}^{\natural} \in \mathbb{R}^{N \times N}$ such that $P^{\natural}_{ij} = \langle \bm{f}^{\natural}(\bm{x}_i), \bm{f}^{\natural}(\bm{x}_j) \rangle$,  
$\bm{P}^{\star } \in \mathbb{R}^{N \times N}$ such that $P^{\star }_{ij} = \langle \bm{f}^{\star }(\bm{x}_i), \bm{f}^{\star }(\bm{x}_j) \rangle$, and
$\bm{S}_{\bm{X}} = [\bm{x}_{i_1}, \bm{x}_{j_1}, \ldots , \bm{x}_{i_M}, \bm{x}_{i_M}]$.
We first show that the matrix $\bm{P}^{\natural}$ is approximately recoverable via minimizing \eqref{eq:logistic_cc}:
\begin{lemma}
    \label{lemma:recovery_P}
    Let $S = \set{(i_1, j_1, y_1), \ldots , (i_M, j_M, y_M)}$, where $(i_1, j_1), \ldots , (i_M, j_M)$ are drawn independently and uniformly at random from $[N] \times [N]$. Suppose that ${y_m \mid (i_m, j_m) \sim {\sf Bernoulli}(P^{\natural}_{i_m, j_m})}$ following the generative model in \eqref{eq:generative_model} and that ${\cal F}$ satisfies Assumption~\ref{assumption:F_cond1}.
    Then,
with probability at least $1-\delta$, we have
\[
\dfrac{1}{N^2}\norm{\bm{P}^{\star} - \bm{P}^{\natural} }_{\rm F}^2  
\leq \epsilon( M, \delta)^2, 
\] 
where $\epsilon( M, \delta)^2 $ is defined as follows:
\begin{equation}\label{eq:eP}
\begin{aligned}
&\epsilon( M, \delta)^2 \triangleq 
\dfrac{64\log(1/\alpha)}{M} + \dfrac{ 96\sqrt{2} \log M}{\alpha M \log 2} \norm{\bm{S}_{\bm{X}}}_{\rm F} \sqrt{R_{\texttt{NET}}}  \\
&+ 64\log (1/\alpha) \sqrt{\dfrac{2 \log(4/\delta)}{M}} + \dfrac{16\nu}{\alpha}.
\end{aligned}
\end{equation}
\end{lemma}
The proof of Lemma~\ref{lemma:recovery_P} is relegated to Appendix~\ref{app:proof_lemma_recovery_P}. It is not surprising that $\bm P^\natural$ can be recovered from minimizing the logistic loss, as the problem of recovering $\bm P^\natural$ can be regarded as a generalized 1-bit matrix completion (MC) problem. Unlike the conventional 1-bit MC frameworks that leveraged the low-rank structure of the complete data (see, e.g., \cite{davenport20141}), here we exploit the neural generative model in \eqref{eq:generative_model}, which is also a low-dimensional model, as long as $R_{\texttt{NET}}$ is sufficiently small (compared to $M$ and $N$).

Before proceeding to showing recovery of $\bm{M}^{\natural}$, let us observe the following fact based on Lemma~\ref{lemma:recovery_P}:
It can be seen in \eqref{eq:eP} that $\epsilon(M, M^{-0.5})^2$ is decreasing with a rate of $\mathcal{O}(\sqrt{(\log M) / M})$, hence there exists $M_0 \in \mathbb{N}$ independent to $N$ such that $\; \forall  M > M_0$,
\begin{equation}
\label{eq:condition_for_M}
\epsilon'(M)^2 \triangleq M^{0.25} \epsilon(M, M^{-0.5})^2 + M^{-0.25} \leq \dfrac{1}{8K^2}.
\end{equation}
Using the above fact, we show that
\begin{theorem}
    \label{theorem:recovery_M_given_P}
    Under the same assumptions as in Lemma~\ref{lemma:recovery_P},
    further assume that Assumption~\ref{assumption:asc} is satisfied.
    Let ${\bm{M}^{\star } = [\bm{f}^{\star }(\bm{x}_1), \ldots , \bm{f}^{\star }(\bm{x}_N)]}$, and consider $M > M_0$, then
      there exists a permutation matrix $\boldsymbol \Pi^{\star }$  such that
\begin{multline*}
\dfrac{1}{NK}\norm{\boldsymbol \Pi^{\star } \bm{M}^{\star } - \bm{M}^{\natural}}_{\rm F}^2
\leq \frac{4N}{K} \epsilon(M, \delta)^2 \\ 
+ \dfrac{2K}{N}(1+8\sigma_{\max}^2(\bm{V}^{\natural})) \epsilon'(M),
\end{multline*}
holds with probability at least  $1-\delta - K^2/M^{0.25}$, where $\epsilon'(M)$ is defined in \eqref{eq:condition_for_M}.
\end{theorem}
The proof of Theorem~\ref{theorem:recovery_M_given_P} is in Appendix~\ref{app:proof_theorem_recovery_M_given_P}.
We should mention that the sample complexity in Theorem~\ref{theorem:recovery_M_given_P} is based on a worst-case analysis,
and thus the bound tends to be pessimistic---it starts to make sense when 
 $$N\leq {\cal O}(K\sqrt{M}),$$ as reflected in the first term on the right hand side.
 In practice, ${\sf Loss}_{\rm cc}$ can work fairly well using a much smaller $M$.
 Nonetheless, Theorem~\ref{theorem:recovery_M_given_P} for the first time shows the soundness of using ${\sf Loss}_{\rm cc}$ in \eqref{eq:logistic_cc}, in terms of being able to guarantee the identifiability of $\M^\natural$. 

With Theorem~\ref{theorem:recovery_M_given_P}, it is readily to show the following:
\begin{theorem}[Generalization]
\label{theorem:generalization}
Under the same conditions as in Theorem~\ref{theorem:recovery_M_given_P}, with probability at least $1 - \delta - K^2/M^{1/4}$,
\begin{multline*}
\mathop{\mathbb{E}}_{\bm{x} \sim \mathcal{P}_{\mathcal{X}}} \left[ \frac{1}{K} \norm{ \boldsymbol \Pi^{\star }\bm{f}^{\star }(\bm{x}) - \bm{f}^{\natural}(\bm{x})}^2 \right] 
\leq 
\frac{4N}{K} \epsilon(M, \delta)^2 \\ 
+ \dfrac{2K}{N}(1+8\sigma_{\max}^2(\bm{V}^{\natural})) \epsilon'(M) \\
+ \dfrac{4}{NK} + \dfrac{12\sqrt{2 R_{\texttt{NET}}} \norm{\bm{X}}_{\rm F} \log N}{NK \log 2} + 8 \sqrt{\dfrac{2 \log (4/\delta)}{NK^2}}.
\end{multline*} 
\end{theorem}
See the proof in Appendix~\ref{app:proof_theorem_generalization}. 
Theorem~\ref{theorem:generalization} confirms that the learned $\bm f^\star$ via minimizing ${\sf Loss}_{\rm cc}$ can generalize to classify unseen data, which was observed in the literature \cite{hsu2018probabilistic,zhang2019framework,zhang2021framework} but never formally shown.

\paragraph{Enhanced Identifiability with Large Sample.}
In Theorem~\ref{theorem:recovery_M_given_P}, the identifiability of $\bm M^\natural$ was established under finite samples. The proof relies on the ASC.
One may argue that the anchor samples may not exist in some cases, as some data naturally exhibit mixed membership, e.g., social network users \cite{airoldi2008mixed} and multi-topic documents  \cite{blei2003latent}.
Here, we relax the finite sample assumption to show that at the limit of large $M$, the logistic loss enjoys enhanced membership identifiability that does not hinge on the ASC.

To see this, let us introduce the following condition:

\begin{assumption}[Sufficiently Scattered Condition (SSC)]\label{assumption:ssc}
Let $\bm{m}_n^{\natural} = \bm{f}^{\natural}(\bm{x}_n)$, $\bm{M}^{\natural} = [\bm{m}^{\natural}_1, \ldots , \bm{m}^{\natural}_N]$. $\bm{M}^{\natural}$ satisfies SSC if there exists a subset ${\cal S}\in [N]$ such that
\[     {\cal C} \subseteq { {\rm cone}}(\bm M^\natural[:,{\cal S}]),  \]
where ${\cal C}=\set{\bm{x} \in \mathbb{R}^{K} \mid \sqrt{K-1} \norm{\bm{x}} \leq \bm{1}^{\T}\bm{x}}$,
and ${\rm cone}(\bm{M}^{\natural}[:, \mathcal{S}]) \subseteq {\rm cone}(\bm{Q})$ does not hold for any orthogonal $\bm{Q}$ except for the permutation matrices.
\end{assumption}
\begin{figure}[t]
    \centering
    \def\svgwidth{0.8\columnwidth}
\begingroup%
  \makeatletter%
  \providecommand\color[2][]{%
    \errmessage{(Inkscape) Color is used for the text in Inkscape, but the package 'color.sty' is not loaded}%
    \renewcommand\color[2][]{}%
  }%
  \providecommand\transparent[1]{%
    \errmessage{(Inkscape) Transparency is used (non-zero) for the text in Inkscape, but the package 'transparent.sty' is not loaded}%
    \renewcommand\transparent[1]{}%
  }%
  \providecommand\rotatebox[2]{#2}%
  \newcommand*\fsize{\dimexpr\f@size pt\relax}%
  \newcommand*\lineheight[1]{\fontsize{\fsize}{#1\fsize}\selectfont}%
  \ifx\svgwidth\undefined%
    \setlength{\unitlength}{502.46745913bp}%
    \ifx\svgscale\undefined%
      \relax%
    \else%
      \setlength{\unitlength}{\unitlength * \real{\svgscale}}%
    \fi%
  \else%
    \setlength{\unitlength}{\svgwidth}%
  \fi%
  \global\let\svgwidth\undefined%
  \global\let\svgscale\undefined%
  \makeatother%
  \begin{picture}(1,0.75176762)%
    \lineheight{1}%
    \setlength\tabcolsep{0pt}%
    \put(0,0){\includegraphics[width=\unitlength,page=1]{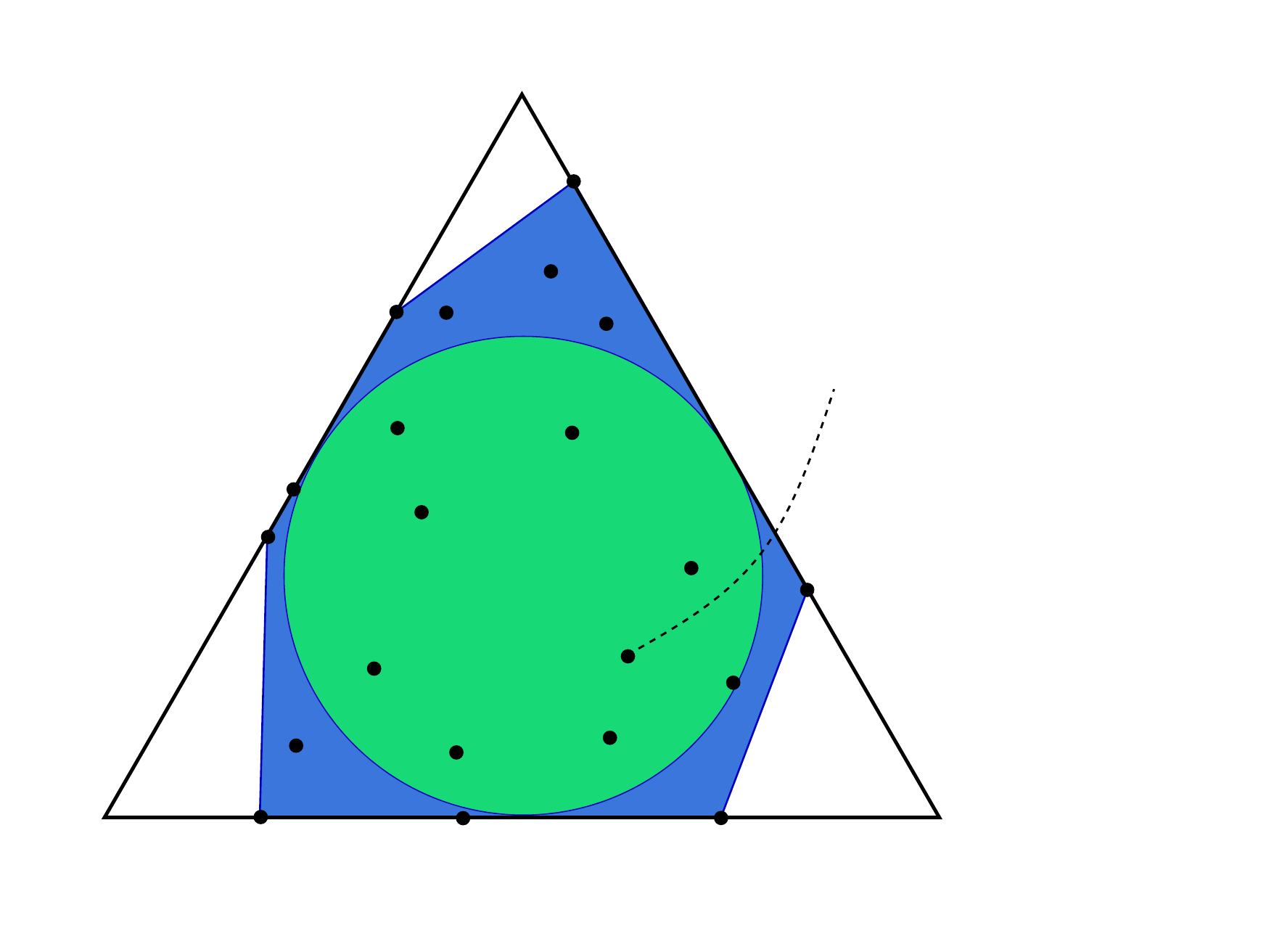}}%
    \put(0.66530726,0.4562635){\makebox(0,0)[lt]{\lineheight{1.25}\smash{\begin{tabular}[t]{l}Membership $\bm{m}^\natural_n$'s\end{tabular}}}}%
    \put(0.0087693,0.61898041){\makebox(0,0)[lt]{\lineheight{1.25}\smash{\begin{tabular}[t]{l}${\rm cone}(\bm{M}^\natural[:, {\cal S}])$\end{tabular}}}}%
    \put(0.03386936,0.31400594){\makebox(0,0)[lt]{\lineheight{1.25}\smash{\begin{tabular}[t]{l}${\cal C}$\end{tabular}}}}%
    \put(0,0){\includegraphics[width=\unitlength,page=2]{ssc.pdf}}%
  \end{picture}%
\endgroup%

    \caption{An example of $\bm{M}^{\natural }$ satisfying SSC when $K=3$.}
    \label{fig:ssc}
\end{figure}
The SSC is illustrated in Fig.~\ref{fig:ssc} with a $K=3$ case. The triangle represents the nonnegative orthant and the green circle the second-order cone ${\cal C}$. SSC means that the conic hull of a subset of the $\bm m_n^\natural$'s (the blue region) covers the green region as a subset. This means that the columns of $\bm M^\natural$ are sufficiently different. The SSC is widely used in the nonnegative matrix factorization literature; see \cite{fu2018identifiability,fu2015blind}. Note that $\bm M^\natural$ satisfying the SSC is much more relaxed compared to the ASC, which means that ${\rm cone}(\bm M^\natural)=\mathbb{R}_+^K$, i.e., the blue region in Fig.~\ref{fig:ssc} covers the entire triangle. In the context of membership learning, an SSC-satisfying $\M^\natural$ means that the membership of the data should contain enough diversity, but an ASC-satisfying $\bm M^\natural$ means that some single-membership data samples need to exist---which is more stringent.

Using the SSC, we show that the following holds:
\begin{theorem}
\label{theorem:identifiability_limit}(Enhanced Identifiability)
    Under the same assumptions as in Lemma~\ref{lemma:recovery_P},
    suppose that $\bm{M}^{\natural}$ satisfies SSC (Assumption~\ref{assumption:ssc}) and that ${\bm f^\natural}\in {\cal F}$. Then, at the limit of $\max \big(\log(1/\alpha), \log (M)\sqrt{R_{\texttt{NET}}}/\alpha \big)/\sqrt{M} \rightarrow 0$, the following statements hold:
    \begin{itemize}
        \item[(i)] There exists a permutation matrix $\boldsymbol \Pi^{\star }$ such that   $
   \boldsymbol \Pi^{\star } \bm{M}^{\star } = \bm{M}^{\natural}.
    $ 
    \item[(ii)] The learned neural network $\bm f^\star$ satisfies
    \begin{multline*}
    \mathop{\mathbb{E}}_{\bm{x} \sim \mathcal{P}_{\mathcal{X}}} \left[ \frac{1}{K} \norm{ \boldsymbol \Pi^{\star }\bm{f}^{\star }(\bm{x}) - \bm{f}^{\natural}(\bm{x})}^2 \right] \leq \dfrac{8}{NK} + \\
    \dfrac{12\sqrt{2 R_{\texttt{NET}}} \norm{\bm{X}}_{\rm F} \log N}{NK \log 2}
+ 8 \sqrt{\dfrac{2 \log (4/\delta)}{NK^2}}
    \end{multline*} 
with probability at least $1-\delta$.
    \end{itemize}
\end{theorem}
The proof of Theorem~\ref{theorem:identifiability_limit} is in Appendix~\ref{app:proof_theorem_identifiability_limit}. Theorem~\ref{theorem:identifiability_limit} has plausible implications in practice: When the sample size $M$ is large and ${\cal F}$ is expressive, even no anchor samples exist, identifying $\M^\natural$ and finding a generalizable $\bm f^\star$ are still possible.

\section{DCC with Noisy Labels: Geometric Regularization}
\paragraph{Incorporating Annotation Confusion.}
We should mention that the generative model in \eqref{eq:generative_model} can already model annotation noise to a certain extent:
The Bernoulli sampling process can encode some 0-1 flipping probability of observing $y_m$.
Nonetheless, this level of noise consideration is not enough, as annotations could be grossly inaccurate.
To take more severe annotation errors into consideration,
we modify the generative model as follows.
We assume that the annotator confuses class $i$ with class $j$ with probability ${\sf Pr}(j|i )$. Let $A_{i,j}={\sf Pr}(i|j)$, we have a ``confusion matrix'' $\A\in\mathbb{R}^{K\times K}$ where $[\A]_{i,j}=A_{i,j}$. Hence, the annotator's ``confused membership vector'' is modeled as
\begin{align}
    {\bm m}_n^{\rm confused} =\A\m_n^\natural= \A \bm f^\natural(\x_n).
\end{align} 
Then, the annotator's output is sampled from the following:
\begin{align}\label{eq:confusedym}
    y_{m} \sim  {\sf Bernoulli}( \langle  \A\bm f^\natural(\x_{ i_m}),\A\bm f^\natural(\x_{ j_m}) {\rangle} ).
\end{align}
Note that using a confusion matrix to model noisy labels' generating process is widely seen in noisy label learning---but mostly under the supervised learning setting; see, e.g., \cite{dawid1979maximum,liu2012variational,zhang2014spectral,chu2021learning}. 
We argue that this confusion model is also suitable for pairwise annotation.
The rationale is that the error happened in comparison is mainly caused by the annotator's confusion on the membership of $\x_{i_m}$ {\it or} that of $\x_{ j_m}$---which is exactly reflected in \eqref{eq:confusedym}.

\paragraph{Volume Maximization DCC.}
To proceed, we propose the following modified logistic loss:
\begin{align*}    
{\sf Loss}'_{\rm cc}(\bm \theta,\bm B)=\dfrac{1}{M}\sum^{M}_{m=1} \Big( y_m \log \dfrac{1}{\bm{f}_{\boldsymbol \theta}(\bm{x}_{i_m})^{\T} \bm{B} \bm{f}_{\boldsymbol \theta}(\bm{x}_{j_m})} + \\
(1-y_m) \log \dfrac{1}{1- \bm{f}_{\boldsymbol \theta}(\bm{x}_{i_m})^{\T} \bm{B} \bm{f}_{\boldsymbol \theta}(\bm{x}_{j_m})} \Big),
\end{align*}
where $\bm{B} \in \mathbb{R}^{K \times K}$ satisfies $ 0\leq \bm{B} \leq 1$, as it is induced by $\bm B=\bm{A}^{\T}\bm{A}$.
Note that we will use $\bm \theta$ and $\bm B$ as our optimization variables (instead of $\bm A$) as it simplifies the loss function.
In addition, as our ultimate goal is to learn $\M^\natural$ and $\bm f^\natural$, the intermediate variable $\A$ does not need to be explicitly estimated.

We show that minimizing ${\sf Loss}'_{\rm cc}(\bm \theta,\bm B)$ provably recovers the data membership and finds a generalizable $\bm f^\star$, with an additional volume requirement satisfied by the solution.
Specifically, we have the following theorem:
\begin{theorem}[Identifiability of Noisy Case]
\label{theorem:identifiability_limit_reg} 
    Assume that the assumptions in Lemma~\ref{lemma:recovery_P} hold, except that the generative model is replaced by \eqref{eq:confusedym}.
   Suppose that $\bm{M}^{\natural}$ satisfies SSC (Assumption~\ref{assumption:ssc}) and that ${\bm f^\natural}\in {\cal F}$. Also assume that ${\rm rank}(\bm{A}^{\T}\bm{A}  \bm{M}^\natural) = K$.
Denote 
\begin{equation}
    \label{eq:def_optimal_theta_B}
    (\bm \theta^\star, \bm B^\star)=\arg\min~{\sf Loss}'_{\rm cc}(\bm \theta,\bm B)
\end{equation} 
and $\bm f^\star = \bm f_{\bm \theta^\star}$ and $\bm m^\star_n=\bm f^\star(\x_n).$
In addition, assume that ${\sf Loss}'_{\rm cc}$ is minimized with a solution $\bm{M}^{\star }, \bm{B}^{\star }$ such that $\log\det(\bm{M}^\star (\bm{M}^\star)^{\T})$ is maximized among all possible optimal solutions.
Then, at the limit of $\max \big(\log(1/\alpha), \log (M)\sqrt{R_{\texttt{NET}}}/\alpha \big)/\sqrt{M} \rightarrow 0$, the following statements hold:
    \begin{itemize}
        \item[(i)] There exists a permutation matrix $\boldsymbol \Pi^{\star }$ such that   $
   \boldsymbol \Pi^{\star } \bm{M}^{\star } = \bm{M}^{\natural}.
    $ 
    \item[(ii)] The learned neural network $\bm f^\star$ satisfies
    \begin{multline*}
    \mathop{\mathbb{E}}_{\bm{x} \sim \mathcal{P}_{\mathcal{X}}} \left[ \frac{1}{K} \norm{ \boldsymbol \Pi^{\star }\bm{f}^{\star }(\bm{x}) - \bm{f}^{\natural}(\bm{x})}^2 \right] \leq \dfrac{8}{NK} + \\
    \dfrac{12\sqrt{2 R_{\texttt{NET}}} \norm{\bm{X}}_{\rm F} \log N}{NK \log 2}
+ 8 \sqrt{\dfrac{2 \log (4/\delta)}{NK^2}}
    \end{multline*} 
with probability at least $1-\delta$.
    \end{itemize}
\end{theorem}

The proof of Theorem~\ref{theorem:identifiability_limit_reg} is in Appendix~\ref{app:proof_of_identifiability_limit_reg}.
The take-home point is that, when one has large samples and an expressive neural network, the membership identifiability and generalization performance of using ${\sf Loss}_{\rm cc}'$ can be as good as that of using ${\sf Loss}_{\rm cc}$---as if there is no annotation confusion.

Of course, there are more requirements to satisfy under Theorem~\ref{theorem:identifiability_limit_reg}.
Particularly, the maximal $\log\det(\M^\star(\M^\star)^\T)$ requirement is nontrivial. In practice, the optimization criterion in Theorem~\ref{theorem:identifiability_limit_reg} can be approximated via a regularized version of ${\sf Loss}_{\rm cc}'$ as follows:
\begin{equation}
    \label{eq:final_formula}
    \minimize_{\bm \theta, \bm 0\leq \bm{B}\leq \bm 1}~{\sf Loss}'_{\rm cc} +   {\sf Loss}_{\rm vol},
\end{equation}
where 
${\sf Loss}_{\rm vol} = - \lambda \log\det(\bm M \bm M^\T)$ and $\bm m_n =\bm f_{\bm \theta}(\x_n)$.
Note that $\log\det(\bm M \bm M^\T)$ is proportional to the volume of the Gram matrix $\bm{M}\bm{M}^{\T}$ \cite{boyd2004convex}. Hence, the term can be regarded as a geometry-driven regularization.
We name our method using \eqref{eq:final_formula} as the
{\it Volume Maximization-Regularized Deep Constrained Clustering} (\texttt{VolMaxDCC}). An overall architecture is shown in Fig.~\ref{fig:architecture}.

\begin{figure}[t]
    \centering
    \def\svgwidth{0.9\columnwidth}
    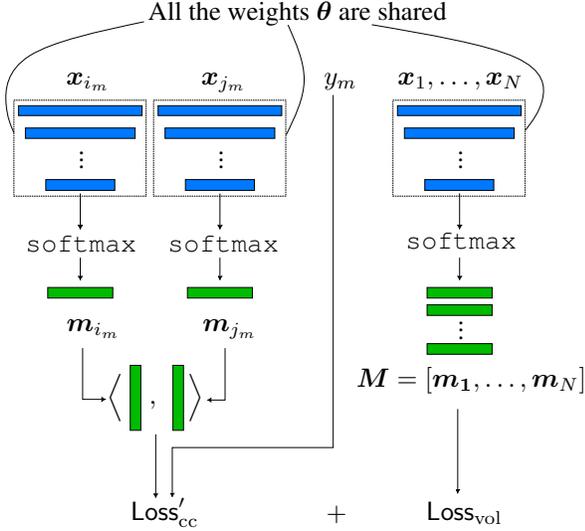
    \caption{The architecture of the \texttt{VolMaxDCC} approach.}
    \label{fig:architecture}
\end{figure}

\section{Related Work}
Recovering the underlying unseen matrix from incomplete and binary measurement is related to 1-bit low-rank matrix completion \cite{davenport20141}, which was often studied in the context of recommender systems. The generative model in \eqref{eq:generative_model} is reminiscent of the MMSB \cite{airoldi2008mixed,huang2019detecting} that has been a workhorse in overlapped community detection. 
The MMSB model with missing links was also used for clustering-related tasks, e.g., crowdclustering \cite{gomes2011crowdclustering}.
The ASC and SSC are commonly seen conditions in identifiability analysis of nonnegative matrix factorization \cite{fu2018identifiability,fu2015self,fu2019nonnegative,fu2016robust,huang2014non,donoho2003does,gillis2014and,gillis2014fast,gillis2014robust,gillis2020nonnegative,nguyen2022memory}.
Using volume maximization/minimization to enhance NMF identifiability appeared as early as 1994 \cite{craig1994minimum} in the context of a blind source separation (BSS) problem in spectral image analysis.
The volume-based geometric regularization was connected to ASC- and SSC-like conditions (e.g., the so-called ``pure pixel condition'' and ``local dominance condition'') in \cite{chan2009convex} and \cite{fu2015blind,fu2016robust,lin2015identifiability}, respectively, to attain uniqueness of matrix factorization models, again, in the context of BSS.
All these models do not involve nonlinear function learning or deep neural networks.
In addition, the classic geometric factorization models were developed with continuous low-rank matrix data---instead of binary data generated from models involving complex {\it unknown} nonlinear function.
Confusion matrices are often used in supervised noisy label learning and crowdsourcing \cite{dawid1979maximum,zhang2014spectral,rodrigues2018deep}, to model probabilistic label transition in the annotating process. 
The ASC and SSC were also used to establish identifiability in supervised (crowdsourced) noisy label learning \cite{xia2019anchor,li2021provably,ibrahim2019crowdsourcing,ibrahim2021crowdsourcing,ibrahim2023deep}.
Incorporating the idea of confusion matrix-based modeling in similarity annotation was not seen before.
The proposed generative model connects label confusion matrices, volume maximization (more generally, ASC/SSC-based identifiable factor analysis), and DCC for the first time, to our best knowledge.

\section{Experiments}%
\label{sec:experiments}

\paragraph{Datasets.} 
We use STL-10 \cite{coates2011analysis}, ImageNet-10 \cite{chang2017deep}, and CIFAR-10 \cite{krizhevsky2009learning}.
For three datasets, we use $N=10000$ samples as the seen data, and set aside $2000$, $2000$ and $45000$ unseen data, respectively, for testing the generalization performance.
In all experiments, $M=10,000$ pairwise constraints are uniformly and randomly drawn from $[N]\times [N]$. 
There are no more than 0.02\% of the total number of pairs annotated for all three datasets.

\paragraph{Baselines.} 
We compare our method with several baselines, including the classic CC methods, i.e., PCKMeans \cite{basu2004active}, COP-KMeans \cite{wagstaff2001constrained}, and the DCC methods, namely, DC-GMM \cite{manduchi2021deep} and C-IDEC \cite{zhang2019framework,zhang2021framework}. We also include the plain-vanilla K-means as a reference. We use a validation set for the baselines whenever proper for parameter tuning and algorithm stopping.
The sizes of the validation sets are $N_{\rm valid}=1000$ for STL-10 and ImageNet-10 and $N_{\rm valid}=5000$ for CIFAR-10.
For C-IDEC, the regularization parameters is chosen from $\set{0, 1{e-}1, 1{e-}2, 1{e-}3, 1{e-}4, 1{e-}5}$. For DC-GMM, we use their heuristic to set the constraint violation penalty with the true (oracle) label flipping rate.
For the proposed \texttt{VolMaxDCC}, we also choose $\lambda$ among $\lambda \in \set{0, 1{e-}1,1{e-}2,1{e-}3,1{e-}4,1{e-}5}$. 
We also include the result of using the simple logistic loss ${\sf Loss}_{\rm cc}$ in \eqref{eq:logistic_cc}, which is referred to as \texttt{VanillaDCC}.

\paragraph{Neural Network Settings.}For all the DCC methods, we employ the unsupervised pre-training method by \cite{li2021contrastive}
to convert the images to feature vectors $\{\x_1,\ldots,\x_N\} {\subseteq \mathbb{R}^{512}}$ \cite{li2021contrastive}. The feature vectors are then fed to a two-hidden-layer fully connected neural network $\bm f_{\bm \theta}(\cdot)$, where each hidden layer has 512 ReLU activation functions. The output layer of $\bm f_{\bm \theta}(\cdot)$ has $K=10$ dimensions with the \texttt{softmax} constraints.
The classic methods also work with the pre-trained feature vectors.

\paragraph{Algorithm Implementation.} 
To tackle the proposed criterion in \eqref{eq:final_formula}, 
we first parameterize $\bm{B}$ such that each element $B_{ij} = 1/(1+\exp(-B'_{ij}))$, where $B'_{ij} \in  \mathbb{R}$ is a trainable parameter. By doing so, \eqref{eq:final_formula} becomes an unconstrained optimization problem. We then employ the commonly used stochastic gradient-based solvers to tackle the re-parameterized problem. In our implementation, we use stochastic gradient descent with a batch size of $128$. We set the learning rate for $\bm{B'}$ and $\boldsymbol \theta$ to be 0.1 and 0.5, respectively. The initialization of $\boldsymbol \theta$ is chosen randomly following uniform distributions with parameters depending on output dimension of each layer. To initialize $\bm{B}'$, we make the diagonal elements to be $1$ and the other elements $-1$. The baselines are handled by their respective author-provided code for optimization.

\subsection{Noisy Machine Annotations.}

\paragraph{Noisy Annotation Settings.}
We first conduct experiments under noise-controlled settings.
To be specific, we divide the experiments into two cases, where the annotations are accurate and noisy, respectively. In the former case, the annotations are set to $1$ if the pair of data are from the same class, and set to $0$ otherwise. In the latter case, to generate noisy pairwise constraints, several supervised machine annotators (i.e., classifiers) are trained using different number of training samples. By doing this, the trained classifiers have different prediction errors. The pairs are then annotated based on the class predictions made by these imperfect classifiers.
If the pair of samples share the same predicted membership by the machine annotator, the annotation is set to be 1 (and 0 otherwise).
The annotations are noisy as the machine annotators are far from perfect.
The annotation noise level can be controlled by tuning the prediction error of the classifiers, via using various amounts of training samples.
The statistics of the annotation errors are provided with our experiment results (see the ``noise level'' column in the tables).

\paragraph{Results.}

\begin{table}[t]
    \centering
    \caption{Clustering performance of (seen data, unseen data) on STL10; $N_{\rm unseen}=2000$.}
    \label{table:stl10}
    \resizebox{\linewidth}{!}{\Huge
        \begin{tabular}{|c|c||c|c|c|c|c|c|c|}
            \hline
            \multicolumn{2}{|c||}{\diagbox[width=10em]{Noise \\ level}{Methods}} & Kmeans & COP-Kmeans & PCKmeans & DC-GMM & C-IDEC & \texttt{VanillaDCC} & \texttt{VolMaxDCC} \\ 
            \hline \hline
            \multirow{3}{*}{0.0\%} 
           & ACC & 0.71, --- & 0.66, --- & 0.70, --- & 0.88, 0.87 & 0.89, 0.88 & \textbf{0.93}, \textbf{0.91} & {\blue 0.91}, {\blue 0.89} \\ 
           & NMI & 0.75, --- & 0.67, --- & 0.72, --- & {\blue 0.82}, {\blue 0.80} & 0.81, {\blue 0.80} & \textbf{0.84}, \textbf{0.81} & {\blue 0.82}, {\blue 0.80} \\ 
           & ARI & 0.54, --- & 0.52, --- & 0.55, --- & 0.80, 0.78  & 0.79, 0.77 & \textbf{0.85}, \textbf{0.83} & {\blue 0.84}, {\blue 0.82} \\ \cline{1-9}
           \multirow{3}{*}{8.3\%} & 
             ACC & --- & 0.70, --- & 0.70, --- & 0.75, 0.76 & 0.77, 0.79 & {\blue 0.78}, {\blue 0.80} & \textbf{0.80}, \textbf{0.81} \\ 
           & NMI & --- & 0.64, --- & \textbf{0.71}, --- & {\blue 0.67}, \textbf{0.69} & {\blue 0.67}, \textbf{0.69} & 0.59, 0.62 &  0.64, {\blue 0.65} \\ 
           & ARI & --- & 0.51, --- & 0.56, --- & 0.57, 0.59 & 0.59, 0.61 & {\blue 0.69}, {\blue 0.71} & \textbf{0.73}, \textbf{0.74} \\ \cline{1-9}
           \multirow{3}{*}{10.3\%} & 
             ACC & --- & 0.62, --- & 0.69, --- & 0.70, 0.72 & 0.70, 0.71 & {\blue 0.72}, {\blue 0.73} & \textbf{0.79}, \textbf{0.81} \\ 
           & NMI & --- & 0.59, --- & \textbf{0.73}, --- & 0.62, {\blue 0.64} &0.60, 0.62 & 0.50, 0.51 & {\blue 0.68}, \textbf{0.70} \\ 
           & ARI & --- & 0.44, --- & 0.55, --- & 0.51, 0.52 & 0.50, 0.52 & {\blue 0.62}, {\blue 0.64} & \textbf{0.77}, \textbf{0.78} \\ \cline{1-9}
           \multirow{3}{*}{15.0\%} & 
             ACC & --- & 0.62, --- & {\blue 0.64}, --- & 0.60, {\blue 0.61} & 0.57, 0.57 & 0.56, 0.58 & \textbf{0.79}, \textbf{0.81} \\ 
           & NMI & --- & 0.54, --- & \textbf{0.72}, --- & 0.54, {\blue 0.55} & 0.50, 0.50 & 0.33, 0.35 & {\blue 0.68}, \textbf{0.69} \\ 
           & ARI & --- & 0.41, --- & {\blue 0.52}, --- & 0.38, 0.39 & 0.38, 0.39 & 0.50, {\blue 0.51} & \textbf{0.76}, \textbf{0.77} \\ \hline
        \end{tabular}
    }
\end{table}
\begin{table}[t]
    \centering
    \caption{Clustering performance of (seen data, unseen data) on CIFAR10; $N_{\rm unseen}=45000$.}
    \label{table:cifar10}
    \resizebox{\linewidth}{!}{\Huge
        \begin{tabular}{|c|c||c|c|c|c|c|c|c|}
            \hline
            \multicolumn{2}{|c||}{\diagbox[width=10em]{Noise \\ level}{Methods}} & Kmeans & COP-Kmeans & PCKmeans & DC-GMM & C-IDEC & \texttt{VanillaDCC} & \texttt{VolMaxDCC} \\ 
            \hline \hline
            \multirow{3}{*}{0.0\%} 
           & ACC & 0.78, --- & 0.67, --- & 0.67, --- & {\blue 0.91}, {\blue 0.89} &0.90, {\blue 0.89} &\textbf{0.92}, \textbf{0.90} & {\blue 0.91}, \textbf{0.90}\\ 
           & NMI & 0.71, --- & 0.66, --- & 0.71, --- & {\blue 0.83}, \textbf{0.81}& {\blue 0.83}, \textbf{0.81}&\textbf{0.84}, {\blue 0.80} & {\blue 0.83}, {\blue 0.80} \\ 
           & ARI & 0.62, --- & 0.54, --- & 0.55, --- &0.82, 0.79 & 0.81, 0.79& \textbf{0.85}, {\blue 0.81}& {\blue 0.84}, \textbf{0.82} \\ \cline{1-9}
           \multirow{3}{*}{4.9\%} & 
             ACC & --- & {\blue 0.75}, --- & 0.70, --- &\textbf{0.86}, \textbf{0.86} &\textbf{0.86}, \textbf{0.86} & \textbf{0.86}, \textbf{0.86} & \textbf{0.86}, \textbf{0.86} \\ 
           & NMI & --- & 0.69, --- & 0.69, --- &\textbf{0.77}, \textbf{0.77} & \textbf{0.77}, \textbf{0.77}& {\blue 0.73}, 0.73 & {\blue 0.73}, {\blue 0.74}\\ 
           & ARI & --- & 0.60, --- & 0.57, --- &{\blue 0.73}, {\blue 0.74} &{\blue 0.73}, {\blue 0.74} &\textbf{0.77}, \textbf{0.77} & \textbf{0.77}, \textbf{0.77} \\ \cline{1-9}
           \multirow{3}{*}{8.7\%} & 
             ACC & --- & 0.64, ---  & 0.72, --- &{\blue 0.76}, 0.76 & {\blue 0.76}, 0.76&{\blue 0.76}, {\blue 0.77} & \textbf{0.83}, \textbf{0.83} \\ 
           & NMI & --- & 0.63, --- & 0.69, --- &\textbf{0.71}, \textbf{0.71} &{\blue 0.70}, {\blue 0.70} & 0.58, 0.59& {\blue 0.70}, {\blue 0.70} \\ 
           & ARI & --- & 0.49, --- & 0.59, --- &0.58, 0.59 &0.59, 0.60 & {\blue 0.70}, {\blue 0.70}& \textbf{0.75}, \textbf{0.75}\\ \cline{1-9}
           \multirow{3}{*}{10.9\%} & 
             ACC & --- & 0.68, --- & 0.70, --- &{\blue 0.74}, {\blue 0.74} &0.73, 0.73 &0.68, 0.69  & \textbf{0.82}, \textbf{0.82} \\ 
           & NMI & --- & 0.61, --- & \textbf{0.68}, --- &{\blue 0.67}, \textbf{0.68} &0.65, {\blue 0.66} &0.48, 0.49 &  \textbf{0.68}, \textbf{0.68}\\ 
           & ARI & --- & 0.50, --- & 0.57, --- &0.55, 0.55 &0.56, 0.57 & {\blue 0.62}, {\blue 0.63}& \textbf{0.74}, \textbf{0.74} \\ \hline
        \end{tabular}
    }
\end{table}

\begin{table}[t]
    \centering
    \caption{Clustering performance of (seen data, unseen data) on ImageNet10; $N_{\rm unseen}=2000$.}
    \label{table:imagenet10}
    \resizebox{\linewidth}{!}{\Huge
        \begin{tabular}{|c|c||c|c|c|c|c|c|c|}
            \hline
            \multicolumn{2}{|c||}{\diagbox[width=10em]{Noise \\ level}{Methods}} & Kmeans & COP-Kmeans & PCKmeans & DC-GMM & C-IDEC & \texttt{VanillaDCC} & \texttt{VolMaxDCC} \\ 
            \hline \hline
            \multirow{3}{*}{0.0\%} 
           & ACC & {\blue 0.85}, --- & 0.79, --- & 0.70, --- & \textbf{0.97}, \textbf{0.96} &\textbf{0.97}, \textbf{0.96} &\textbf{0.97}, \textbf{0.96} & \textbf{0.97}, \textbf{0.96} \\ 
           & NMI & 0.80, --- & 0.77, --- & 0.75, --- &{\blue 0.93}, {\blue 0.91} &{\blue 0.93}, \textbf{0.92} & \textbf{0.94}, {\blue 0.91} & \textbf{0.94}, {\blue 0.91} \\ 
           & ARI & 0.68, --- & 0.66, --- & 0.55, --- &\textbf{0.94}, \textbf{0.92} &\textbf{0.94}, \textbf{0.92} & {\blue 0.93}, {\blue 0.91} & {\blue 0.93}, {\blue 0.91} \\ \cline{1-9}
           \multirow{3}{*}{3.4\%} & 
             ACC & --- & 0.79, --- & 0.66, --- &{\blue 0.93}, 0.92 &{\blue 0.93}, {\blue 0.93} &0.92, 0.91 & \textbf{0.94}, \textbf{0.94} \\ 
           & NMI & --- & 0.75, --- & 0.73, --- & 0.86, 0.85 & {\blue 0.87}, {\blue 0.86}&0.83, 0.82 & \textbf{0.88}, \textbf{0.87} \\ 
           & ARI & --- & 0.65, --- & 0.52, --- &0.84, 0.83 &{\blue 0.86}, {\blue 0.85} &0.84, 0.84 & \textbf{0.89}, \textbf{0.88} \\ \cline{1-9}
           \multirow{3}{*}{6.9\%} & 
             ACC & --- & 0.74, --- & 0.72, --- &0.84, 0.84 &{\blue 0.88}, {\blue 0.88} & 0.84, 0.84 & \textbf{0.92}, \textbf{0.91} \\ 
           & NMI & --- & 0.70, --- & 0.76, --- &0.79, 0.79 &{\blue 0.82}, {\blue 0.82} & 0.70, 0.70 & \textbf{0.84}, \textbf{0.83} \\ 
           & ARI & --- & 0.58, --- & 0.59, --- & 0.71, 0.71&{\blue 0.77}, {\blue 0.77} & {\blue 0.77}, {\blue 0.77} & \textbf{0.88}, \textbf{0.87} \\ \cline{1-9}
           \multirow{3}{*}{11.2\%} & 
             ACC & --- & 0.72, --- & 0.62, --- &0.71, 0.72 &{\blue 0.80}, {\blue 0.81} &0.65, 0.66 & \textbf{0.91}, \textbf{0.90} \\ 
           & NMI & --- & 0.64, --- & 0.73, --- &0.68, 0.70 & {\blue 0.74}, {\blue 0.76}&0.49, 0.52 & \textbf{0.83}, \textbf{0.82} \\ 
           & ARI & --- & 0.54, --- & 0.51, --- &0.56, 0.58 &{\blue 0.66}, {\blue 0.68} & 0.62, 0.63 & \textbf{0.87}, \textbf{0.86} \\ \cline{1-9}
        \end{tabular}
    }
\end{table}

We report average performance on the seen set $\{\x_1,\ldots,\x_N\}$ and the unseen data over 5 random trials in Tables~\ref{table:stl10}, \ref{table:cifar10}, and \ref{table:imagenet10}---which correspond to STL10, CIFAR10, and ImageNet10, respectively. 
For K-means, COP-Kmeans, and PCKMeans, we only report results on the seen set since these methods do not have the notion of generalization. The performance is measured by three commonly seen metrics, namely, clustering accuracy (ACC)  \cite{cai2010locally}, normalized mutual information (NMI)  \cite{cai2010locally}, and adjusted rank index (ARI) \cite{yeung2001details}. For all metrics, a higher score indicates a better performance.

In case where accurate pairwise constraints are used (i.e., the rows in all the tables corresponding to ``Noise Level = 0\%''), most methods work reasonably well. As expected, all the DCC methods exhibit tangible edges over the CC methods that do not use deep neural networks.
This is consistent with the observations made from previous DCC works \cite{manduchi2021deep,zhang2019framework,zhang2021framework}.
The good performance of \texttt{VanillaDCC} on both training and testing set are also as expected, per our identifiability analysis.

The rows in the tables associated with nonzero noise levels show that the performance of DC-GMM, C-IDEC, and \texttt{VanillaDCC} drops quickly.
For example, in Table~\ref{table:cifar10}, the ACC of DC-GMM drops from (0.91,0.89) to (0.74,0.74) when the noise level changes from 0\% to 10.9\%. Similar performance degradation is observed for C-IDEC and \texttt{VanillaDCC}, which are both DCC methods that do not explicit consider annotation noise.
Nonetheless, the proposed \texttt{VolMaxDCC}'s performance decline is much more graceful on all three datasets.
In particular, Table~\ref{table:imagenet10} shows that the ACC of \texttt{VolMaxDCC} is still at (0.91,0.90) when the noise level reaches 11.2\%, while the baselines have a best ACC of (0.80,0.81).
The results show the usefulness of our confusion model, as well as the effectiveness of our identifiability-driven loss function design.
More experiment results can be seen in Appendix~\ref{app:ablation}.

\begin{table}[t!]
    \centering
    \caption{Clustering performance of (seen data, unseen data) on ImageNet10 with pairwise annotations acquired from the AMT platform; $N_{{\rm unseen}}=2000$, noise level: $23.09$\%.}
    \label{table:imagenet10_cc_real}
    \resizebox{\linewidth}{!}{\Huge
        \begin{tabular}{|c||c|c|c|c|c|c|}
            \hline
            \diagbox[width=10em]{Metrics}{Methods} & Kmeans & COP-Kmeans & PCKmeans & DC-GMM & C-IDEC & \texttt{VolMaxDCC} \\ 
            \hline \hline
           ACC &0.84, --- & 0.68, ---&0.84, ---& 0.87, 0.85& {\blue 0.91}, {\blue 0.90}&\textbf{0.95}, \textbf{0.94} \\ 
           NMI &0.79, --- & 0.49, ---&0.79, ---& {\blue 0.88}, {\blue 0.86}&  0.86, 0.85&\textbf{0.89}, \textbf{0.87}  \\ 
           ARI &0.67, --- & 0.42, ---&0.67, ---& 0.82, 0.79&  {\blue 0.83}, {\blue 0.82}&\textbf{0.89}, \textbf{0.88}  \\  \hline
        \end{tabular}
    }
\end{table}

\subsection{\bf Noisy AMT Annotations.} 
\paragraph{\bf Data Acquisition.} In addition to using machine classifier-annotated data, we also conduct experiments using pairwise annotations that are obtained from the AMT platform. We uploaded $8994$ data pairs to AMT, where the samples are from the ImageNet10 dataset. The annotators were asked to provide their judgement on the similarity of the pairs. 
Recall that there are $N=10000$ samples in the ImageNet10 dataset, which means that $0.018$\% of all the pairs were annotated.
We manually checked the error rate, which was found to be 23.09\%.
The annotated pairs are also released with the code. 

\paragraph{Results.}
Table~\ref{table:imagenet10_cc_real} shows the results on this AMT dataset.
As before, we use the available pairs to learn the membership of training data and observe the testing accuracy over $N_{\rm unseen}=2,000$ samples.
One can see that the proposed \texttt{VolMaxDCC} exhibits the highest clustering accuracy over the seen and unseen data. The clustering accuracy of the second best baseline is 4\% lower than that of \texttt{VolMaxDCC}
over both seen and unseen data. The margins of the proposed method over the baselines in terms of NMI and ARI are also obvious.
The performance on the noisy AMT data speaks for the usefulness and effectiveness of the proposed method in real-world scenarios.

\section{Conclusion}%
\label{sec:conclusion}
We revisited the pairwise annotation-based DCC problem from a membership identifiability viewpoint.
We showed that a recently emerged logistic DCC loss is a sound criterion in terms of model identification---if the annotations are generated following a model that is reminiscent of the MMSB and deep learning based classifier learning.
Based our understanding to the vanilla logistic loss, we moved forward to consider the noisy annotation case and proposed
a confusion-matrix based generative model. We proposed a modified logistic loss with a geometric regularization for provable membership identification---whose identifiability guarantee is the first of the kind under noisy annotations-based DCC, to our best knowledge. We tested our new design over various datasets under multiple noisy levels. We observed tangible improvements over all cases, showing our confusion-based modeling and identifiability-driven design are promising.

\paragraph{Limitations.}
The proposed approach has a couple of notable limitations. First, the model used for annotator noise relies on a confusion matrix model, assuming uniform confusion across all data samples, which may not always hold true in practical scenarios. Developing a framework that takes into account more realistic confusion models could lead to further improvements in performance. Second, the establishment of membership identifiability under the SSC assumption lacks finite sample analysis (cf. Theorem~\ref{theorem:identifiability_limit} and Theorem~\ref{theorem:identifiability_limit_reg}). The assumption that $M$ reaches infinity can never be met in practice. It is of great interest to show how the performance of the volume-based criterion scales with different sample sizes.

\paragraph{Acknowledgement.} 
This work is supported in part by the National Science Foundation (NSF) under project NSF IIS-2007836.

\bibliography{my_ref}
\bibliographystyle{icml2023}

\newpage
\appendix
\onecolumn

\section{Proof of Lemma~\ref{lemma:recovery_P}}
\label{app:proof_lemma_recovery_P}

Denote $(I, J)$ and $Y$ as two random variables (RVs) such that
$\pr{I, J, Y } = \pr{ (I, J) }\pr{ Y| (I, J)}$ and 
\[
\begin{cases}
    &I, J \sim \mathcal{U}([N] \times [N]),\\
    &Y|(I, J) \sim {\sf Bernoulli}(P^{\natural}_{IJ}),
\end{cases}
\] 
where $(i_1, j_1, y_1), \ldots , (i_M, j_M, y_M)$ are $M$ i.i.d realizations of $(I, J, Y)$.
Denote $\mathcal{D}_{\bm{P}^{\natural}}$ as the joint distribution of $(I, J, Y)$, 
i.e, $(I, J, Y) \sim \mathcal{D}_{\bm{P}^{\natural}}$.
We first note the following relationship, which was also used in 1-bit matrix completion literature \cite{davenport20141,cai2016matrix}:
\begin{align}\label{eq:klP}
\Dkl{\bm{P}^{\natural}}{\bm{P}^{\star}}
&= \dfrac{1}{N^2} \sum_{(i,j) \in [N]^2} \dkl{P^{\natural}_{ij}}{P^{\star}_{ij}} \quad  \text{(by definition in \eqref{eq:dkl_matrix})}  \nonumber\\
&= \mathop{\mathbb{E}}_{(I,J) \sim \mathcal{U}} \left[  \dkl{P_{IJ}^{\natural}}{P^{\star}_{IJ}}\right] \nonumber \\
&= \mathop{\mathbb{E}}_{(I,J) \sim \mathcal{U}} \left[ P^{\natural}_{IJ} \log \dfrac{P^{\natural}_{IJ}}{P^{\star}_{IJ}} + (1-P_{IJ}^{\natural}) \log \dfrac{1-P_{IJ}^{\natural}}{1 - P^{\star}_{IJ}} \right] \quad \text{(by definition in \eqref{eq:dkl_element})} \nonumber\\
&= \mathop{\mathbb{E}}_{(I,J) \sim \mathcal{U}} \left[ \mathop{\mathbb{E}}_{Y|(I,J) \sim \text{Bern($P^{\natural}_{IJ}$)}} \left[  Y \log \dfrac{P^{\natural}_{IJ}}{P^{\star}_{IJ}} + (1-Y) \log \dfrac{1-P^{\natural}_{IJ}}{1 - P^{\star}_{IJ}} \right] \right] \nonumber\\
&= \mathop{\mathbb{E}}_{(I, J, Y) \sim \mathcal{D}_{\bm{P}^{\natural}}} 
\left[ Y\log P^{\natural}_{IJ} + (1-Y) \log P^{\natural}_{IJ} \right] 
- \mathop{\mathbb{E}}_{(I, J, Y) \sim \mathcal{D}_{\bm{P}^{\natural}}} 
\left[ Y \log P^{\star }_{IJ} + (1-Y) \log P^{\star }_{IJ} \right] \nonumber\\
&= L_{\mathcal{D}_{\bm{P}^{\natural}}}(\Pstar) - L_{\mathcal{D}_{\bm{P}^{\natural}}}(\Psharp) ,
\end{align}
where we define $L_{\mathcal{D}_{\Psharp}}(\bm{P})$ for any matrix $0 \leq \bm{P} \leq 1$ as
\begin{alignat*}{2}
    &L_{\mathcal{D}_{\Psharp}}(\bm{P}) && \triangleq \mathbb{E}_{(I, J, Y) \sim \mathcal{D}_{\bm{P}^{\natural}}} \left[ \ell (\bm{P}, (I, J, Y)) \right], \text{ where }\\
    &\ell(\bm{P}, (I, J, Y)) &&\triangleq -Y \log P_{IJ} - (1-Y) \log (1-P_{IJ}). \numberthis \label{eq:loss_function_l} 
\end{alignat*} 
Recall $S \triangleq \set{(i_1, j_1, y_1), \ldots , (i_M, j_M, y_M)}$. 
Define $ L_S(\bm{P}) \triangleq (1/M) \sum_{i=1}^{M} \ell (\bm{P}, (i_m, j_m, y_m))$.
By this definition, $\bm{P}^{\star }$ is an optimal solution of the following problem
\begin{alignat*}{2}
    & \minimize_{\bm{P} \in \mathcal{P}_{\mathcal{F}, \bm{X}}} \quad && L_S(\bm{P}) ,
\end{alignat*}
where $\mathcal{P}_{\mathcal{F}, \bm{X}}$ is a set of matrices defined by the function class $\mathcal{F}$ and the data set $\bm{X}=[\x_1,\ldots,\x_N]$. Specifically, it is defined as
\begin{subequations}
\label{eq:define_class}
\begin{alignat}{2}
&\mathcal{P}_{\mathcal{F}, \bm{X}} &&\triangleq \set{\bm{P}\in \mathbb{R}^{N \times N} \mid  P_{ij} = \texttt{dot}(\bm{f}'(\bm{x}_{i}, \bm{x}_{j})) },  \\
&\texttt{dot}([\bm{u}; \bm{v}]) && \triangleq \bm{u}^{\T}\bm{v}, \label{eq:dot_definition}\\
& \bm{f}'(\bm{x}, \bm{y}) && \triangleq [\bm{f}(\bm{x}); \bm{f}(\bm{y})], \quad \bm{f} \in \mathcal{F}.
\end{alignat}
\end{subequations}

Define $\widetilde{\bm{P}} \in \mathbb{R}^{N\times N}$ such that $\widetilde{P}_{ij} = \widetilde{\bm{f}}(\bm{x}_i)^{\T} \widetilde{\bm{f}}(\bm{x}_j)$ where $\widetilde{\bm{f}}$ is defined in Assumption~\ref{assumption:prox}, i.e., the ``best approximation'' for $\bm f^\natural$ by the class ${\cal F}$.
Using \eqref{eq:klP} and the above notations, we have
\begin{align*}
\Dkl{\Psharp}{\Pstar}
&= L_{\mathcal{D}_{\Psharp}}(\Pstar) - L_{\mathcal{D}_{\Psharp}}(\Psharp) \\
&= \big (L_{\mathcal{D}_{\Psharp}}(\bm{P}^{\star }) - L_{S}(\bm{P}^{\star }) \big) 
+  \big( L_S(\bm{P}^{\star }) - L_S(\widetilde{\bm{P}}) \big)
+ \big( L_S(\widetilde{\bm{P}}) - L_{\mathcal{D}_{\Psharp}}(\widetilde{\bm{P}}) \big) 
+ \big( L_{\mathcal{D}_{\Psharp}}(\widetilde{\bm{P}}) - L_{\mathcal{D}_{\Psharp}}(\bm{P}^{\natural}) \big) \\
&\leq \big (L_{\mathcal{D}_{\Psharp}}(\bm{P}^{\star }) - L_{S}(\bm{P}^{\star }) \big) 
+ \big( L_S(\widetilde{\bm{P}}) - L_{\mathcal{D}_{\Psharp}}(\widetilde{\bm{P}}) \big) 
+ \big( L_{\mathcal{D}_{\Psharp}}(\widetilde{\bm{P}}) - L_{\mathcal{D}_{\Psharp}}(\bm{P}^{\natural}) \big) \\
&\leq 2 \sup_{\bm{P} \in \mathcal{P}_{\mathcal{F}, \bm{X}}} \abs{L_{\mathcal{D}_{\Psharp}}(\bm{P}) - L_S(\bm{P})} + \abs{L_{\mathcal{D}_{\Psharp}}(\widetilde{\bm{P}}) - L_{\mathcal{D}_{\Psharp}}(\bm{P}^{\natural})} , \numberthis \label{eq:kl_excess_risk}
\end{align*}
where the first inequality holds because $L_S(\bm{P}^{\star }) \leq L_S(\widetilde{\bm{P}})$ which holds by the definition of $\bm{P}^{\star }$, and the second inequality holds as both $\bm{P}^{\star }, \widetilde{\bm{P}} \in \mathcal{P}_{\mathcal{F}, \bm{X}}$.

\paragraph{Bound the first term in \eqref{eq:kl_excess_risk}.} 
Notice that $L_{\mathcal{D}_{\bm{P}^{\natural}}}(\bm{P})$ and $L_S(\bm{P})$ are the expected loss and the empirical loss, respectively. Hence, the difference between these two quantities can be bounded using generalization analysis via concentration inequalities. 

To this end, we first determine an upper bound on the loss function $\ell (\bm{P}, (i, j, y))$.
{ Under the assumption stated in Assumption~\ref{assumption:F_cond1}, }
\begin{equation*}
\label{eq:bound_P_W}
\alpha < P_{ij} < 1-\alpha, \quad \forall (i,j) \in [N]\times [N],
\end{equation*} 
which leads to
\begin{equation}
\label{eq:upper_bound_loss}
\ell (\bm{P}, (i, j, y)) 
= y \log \dfrac{1}{P_{ij}} + (1-y) \log  \dfrac{1}{1-P_{ij}} 
\leq \log \dfrac{1}{P_{ij}} + \log \dfrac{1}{1-P_{ij}}
\leq 2 \log \dfrac{1}{\alpha}, \quad \; \forall (i, j) \in [N]^2, y \in \set{0, 1},
\end{equation}

As $\ell (\bm{P}, (i, j, y))$ is bounded, 
the sample set $S$ contains $M$ samples of $(i_m, j_m, y_m)$'s, where each sample $(i_m, j_m, y_m)$  is generated independently with the same distribution $\mathcal{D}_{\bm{P}^{\natural}}$, one can readily apply the following generalization bound \citep[Theorem 26.5]{shalev2014understanding}:
\begin{align*}
\sup_{\bm{P} \in \mathcal{P}_{\mathcal{F}, \bm{X}}} \abs{L_{\mathcal{D}_{\Psharp}}(\bm{P}) - L_{S}(\bm{P})}
&\leq 2 \mathcal{R}(\ell \circ \mathcal{P}_{\mathcal{F}, \bm{X}} \circ S) + 8 \log(1/ \alpha) \sqrt{\dfrac{2\log (4 /\delta)}{M}} \\
&\leq \dfrac{{8}\log(1/\alpha)}{M} + \dfrac{12\sqrt{2} \log M}{\alpha M \log 2} \norm{\bm{S}_{\bm{X}}}_{\rm F}  \sqrt{R_{\texttt{NET}}}+  8\log(1/\alpha) \sqrt{\dfrac{2 \log(4/\delta)}{M}} 
\numberthis \label{eq:bound_excess_risk}
\end{align*} 
holds with probability at least $1-\delta$, where we have applied Lemma~\ref{lemma:rademacher} to reach the final bound.

\paragraph{Bound the second term in \eqref{eq:kl_excess_risk}.} 

We have
\begin{align*}
\abs{L_{\mathcal{D}_{\bm{P}^{\natural}}}(\widetilde{\bm{P}}) - L_{\mathcal{D}_{\bm{P}^{\natural}}}(\bm{P}^{\natural})}
&=  \abs{ \mathop{\mathbb{E}}_{(I, J, Y) \sim \mathcal{D}_{\bm{P}^{\natural}}} \left[ \ell (\widetilde{\bm{P}}, (I, J, Y)) - \ell (\bm{P}^{\natural}, (I, J, Y)) \right]} \\
&= \abs{ \mathop{\mathbb{E}}_{(I, J, Y) \sim \mathcal{D}_{\bm{P}^{\natural}}} \left[Y \left(  \log P^{\natural}_{IJ} - \log \widetilde{P}_{IJ} \right)
+ (1-Y) \left( \log (1-P^{\natural}_{IJ}) - \log ( 1- \widetilde{P}_{IJ} )\right)  \right] } \\
&\leq \mathop{\mathbb{E}}_{(I, J, Y) \sim \mathcal{D}_{\bm{P}^{\natural}}} \left[\abs{  Y \left(  \log P^{\natural}_{IJ} - \log \widetilde{P}_{IJ} \right)
+ (1-Y) \left( \log (1-P^{\natural}_{IJ}) - \log ( 1- \widetilde{P}_{IJ} )\right) } \right]  \\
&\leq \mathop{\mathbb{E}}_{(I, J, Y) \sim \mathcal{D}_{\bm{P}^{\natural}}} \left[Y \abs{ \log \dfrac{P^{\natural}_{IJ}}{\widetilde{P}_{IJ}} } + (1-Y) \abs{\log \dfrac{1-P^{\natural}_{IJ}}{1- \widetilde{P}_{IJ}} } \right]\\
&\leq  \mathop{\mathbb{E}}_{(I, J, Y) \sim \mathcal{D}_{\bm{P}^{\natural}}} \left[ \abs{ \log \dfrac{P^{\natural}_{IJ}}{\widetilde{P}_{IJ}} } + \abs{\log \dfrac{1-P^{\natural}_{IJ}}{1- \widetilde{P}_{IJ}} } \right]    \quad \text{(since $Y \in \set{0, 1}$)}\\
&\leq^{(a)} \mathop{\mathbb{E}}_{(I,J, Y) \sim \mathcal{D}_{\bm{P}^{\natural}}} \left[ \abs{ \dfrac{P^{\natural}_{IJ}- \widetilde{P}_{IJ}}{\widetilde{P}_{IJ}} } + \abs{\dfrac{\widetilde{P}_{IJ}-P^{\natural}_{IJ}}{1- \widetilde{P}_{IJ}} } \right] \\
&\leq \dfrac{1}{\alpha} \mathop{\mathbb{E}}_{(I, J, Y) \sim \mathcal{D}_{\bm{P}^{\natural}}} \left[  \abs{P^{\natural}_{IJ}- \widetilde{P}_{IJ}}\right],  \numberthis \label{eq:qwikzga} 
\end{align*}
where $^{(a)}$ holds by applying inequality $\log(x+1) \leq x, \; \forall x>-1$.

In order to bound \eqref{eq:qwikzga}, consider arbitrary $\bm{x}, \bm{y} \in \mathcal{X}$, 
\begin{align*}
\abs{\bm{f}^{\natural}(\bm{x})^{\T} \bm{f}^{\natural}(\bm{y}) - \widetilde{\bm{f}}(\bm{x})^{\T} \widetilde{\bm{f}}(\bm{y})}
&= \abs{\left(  \bm{f}^{\natural}(\bm{x}) - \widetilde{\bm{f}}(\bm{x}) +\widetilde{\bm{f}}(\bm{x}) \right)^{\T} \left( \bm{f}^{\natural}(\bm{y}) - \widetilde{\bm{f}}(\bm{y}) + \widetilde{\bm{f}}(\bm{y}) \right) - \widetilde{\bm{f}}(\bm{x})^{\T} \widetilde{\bm{f}}(\bm{y})} \\
&= \abs{\left( \bm{f}^{\natural}(\bm{x}) - \widetilde{\bm{f}}(\bm{x}) \right)^{\T} \left( \bm{f}^{\natural}(\bm{y}) - \widetilde{\bm{f}}(\bm{y}) \right)
+ \left( \bm{f}^{\natural}(\bm{x}) - \widetilde{\bm{f}}(\bm{x}) \right)^{\T} \widetilde{\bm{f}}(\bm{y})
+ \widetilde{\bm{f}}(\bm{x})^{\T} \left( \bm{f}^{\natural}(\bm{y}) - \widetilde{\bm{f}}(\bm{y}) \right)} \\
&\leq  \norm{ \bm{f}^{\natural}(\bm{x}) - \widetilde{\bm{f}}(\bm{x})} \norm{ \bm{f}^{\natural}(\bm{y}) - \widetilde{\bm{f}}(\bm{y})} 
+ \norm{\bm{f}^{\natural}(\bm{x}) - \widetilde{\bm{f}}(\bm{x})}  \norm{\widetilde{\bm{f}}(\bm{y})}
+ \norm{\widetilde{\bm{f}}(\bm{x})} \norm{\bm{f}^{\natural}(\bm{y}) - \widetilde{\bm{f}}(\bm{y}) } \\
&\leq  \nu^2 + 2\nu   \quad \text{(by Eq. \eqref{assumption:prox})} \\
&< 4\nu, \quad \text{(since $\nu \leq \sqrt{2}$ due to the fact that all $\bm f(\x)$ are in the probability simplex )  } \numberthis \label{eq:kppp}
\end{align*} 
where we have used $\|\widetilde{\bm f}(\x)\|_2\leq \|\widetilde{\bm f}(\x)\|_1 = 1$ in the second inequality, as the neural network uses a softmax output layer.

Substituting \eqref{eq:kppp} into \eqref{eq:qwikzga}, we have
\begin{equation}
\label{eq:bound_2nd_term}
L_S(\widetilde{\bm{P}}) - L_S(\bm{P}^{\natural})
\leq \dfrac{4\nu }{\alpha}.
\end{equation}
\paragraph{Putting together.} 
Combining \eqref{eq:kl_excess_risk}, \eqref{eq:bound_excess_risk}, \eqref{eq:bound_2nd_term} gives
\[
\Dkl{\bm{P}^{\natural}}{\bm{P}^{\star }} 
\leq \dfrac{16\log(1/\alpha)}{M} + \dfrac{24\sqrt{2} \log M}{\alpha M \log 2} \norm{\bm{S}_{\bm{X}}}_{\rm F}  \sqrt{R_{\texttt{NET}}}+  16\log(1/\alpha) \sqrt{\dfrac{2 \log(4/\delta)}{M}}
+ \dfrac{4 \nu}{\alpha}
\] 

By Lemma~\ref{lemma:fro_to_kl}, we can upper bound our quantity of interest by the following:
\begin{equation}
\label{eq:fro_to_kl}
\norm{\bm{P}^{\star} - \bm{P}^{\natural}}_{\rm F}^2 \leq 4N^2 \Dkl{\bm{P}^{\natural}}{\bm{P}^{\star}}.
\end{equation} 
That is,
\[
\dfrac{1}{N^2}\norm{\bm{P}^{\star} - \bm{P}^{\natural}}_{\rm F}^2 
\leq \dfrac{64\log(1/\alpha)}{M} + \dfrac{ 96\sqrt{2} \log M}{\alpha M \log 2} \norm{\bm{S}_{\bm{X}}}_{\rm F} \sqrt{R_{\texttt{NET}}}  + 64\log (1/\alpha) \sqrt{\dfrac{2 \log(4/\delta)}{M}}  + \dfrac{16\nu}{\alpha}
\] 
hold with probability at least $1-\delta$ over $S = \set{(i_1, j_1, y_1) \ldots , (i_M, j_M, y_M)}$.

\subsection{Supporting lemmas for the proof of Theorem~\ref{lemma:recovery_P}}%

\subsubsection{Rademacher complexity bound for \texorpdfstring{$\mathcal{P}_{\mathcal{F}, \bm{X}}$}{P\_(F, X)}}
\begin{definition}[Rademacher complexity]
    Let $A \subset R^m$ be a set of vectors. We define Rademacher complexity of $A$ as
    \[
    R(A) \triangleq \frac{1}{m} \mathop{\mathbb{E}}_{\boldsymbol \sigma} \left[ \sup_{\bm{a} \in A} \sum_{i=1}^m \sigma_i a_i \right],
    \]
    where $\sigma_1, \ldots, \sigma_m \in \set{-1, 1}$ are i.i.d. distributed according to $\pr{\sigma_i = 1} = \pr{\sigma_i = -1}=0.5$.
\end{definition}
\begin{definition}[Covering number]
    Let $A \subset R^m$ be a set of vectors. We say that $A$ is $r$-covered by a set $A'$, with respect to metric $d$ if for all $\bm{a} \in A$, there exists $\bm{a}' \in A'$ with $d(a, a')\leq r$. We define by $\mathcal{N}(r, A, d)$ the cardinality of the smallest $A'$ that $r$-covers $A$.
\end{definition}
\begin{lemma}[Rademacher complexity bound]
    \label{lemma:rademacher}
    Consider function class $\mathcal{F}$ satisfying Assumptions \ref{assumption:F_cond1} 
, and that $\alpha<\bm{f}(\bm{x})^{\T} \bm{f}(\bm{y})<1-\alpha, \; \forall \bm{x}, \bm{y} \in \mathcal{X}, \forall \bm{f} \in \mathcal{F}$. Then, we have
    \[
    \mathcal{R}(\ell \circ \mathcal{P}_{\mathcal{F}, \bm{X}} \circ S) 
\leq \dfrac{4\log(1/\alpha)}{M} + \dfrac{6\sqrt{2} \log M}{\alpha M \log 2} \norm{\bm{S}_{\bm{X}}}_{\rm F}  \sqrt{R_{\texttt{NET}}},
    \] 
    where
$
\ell \circ \mathcal{P}_{\mathcal{F}, \bm{X}} \circ S \triangleq \set{[\ell(\bm{P}, (i_1, j_1, y_1)), \ldots , \ell(\bm{P}, (i_M, j_M, y_M))] \in \mathbb{R}^{M} \mid \bm{P} \in \mathcal{P}_{\mathcal{F}, \bm{X}} }
$, $\mathcal{R}(A)$ denotes Rademacher complexity of set $A \subseteq \mathbb{R}^M$,  
$\mathcal{P}_{\mathcal{F}, \bm{X}}$ was defined in \eqref{eq:define_class}, 
$S = \set{(i_1, j_1, y_1), \ldots , (i_M, j_M, y_M)}$, and the loss function $\ell$ was defined in \eqref{eq:loss_function_l}.
\end{lemma}
\begin{proof}
We start with deriving the covering number of an $\epsilon$-net of the set $\ell \circ \mathcal{P}_{\bm{X}} \circ S$:
\begin{align*}
\mathcal{N}(\epsilon, \ell \circ \mathcal{P}_{\bm{X}} \circ S, \norm{\cdot}) 
&= \mathcal{N}(\epsilon, \ell \circ \texttt{dot} \circ \bm{f}' \circ S, \norm{\cdot})  \\
&\leq \mathcal{N}(\epsilon /L_{\ell}, \texttt{dot} \circ \bm{f}' \circ S, \norm{\cdot}) \numberthis \label{eq:l_lipschizt} \\
&\leq \mathcal{N}\big(\epsilon / (2L_{\ell}), \bm{f}' \circ S, \norm{\cdot}_{\rm F}\big)  \numberthis \label{eq:dot_lipschitz} \\
&\leq \mathcal{N}\big(\epsilon / (L_{\ell}\sqrt{2}), \bm{f} \circ S_{\bm{X}}^{(1)}, \norm{\cdot}_{\rm F}\big)\mathcal{N}\big(\epsilon / (L_{\ell}\sqrt{2}), \bm{f} \circ S_{\bm{X}}^{(2)}, \norm{\cdot}_{\rm F}\big) \numberthis \label{eq:cover_cartesian_product} \\
&= \mathcal{N}\big(\epsilon / (L_{\ell}\sqrt{2}), \texttt{softmax} \circ \texttt{net}\circ S_{\bm{X}}^{(1)}, \norm{\cdot}_{\rm F}\big)\mathcal{N}\big(\epsilon / (L_{\ell}\sqrt{2}), \texttt{softmax} \circ \texttt{net} \circ S_{\bm{X}}^{(2)}, \norm{\cdot}_{\rm F}\big)  \\
&\leq \mathcal{N}\big(\epsilon / (L_{\ell}\sqrt{2}), \texttt{net} \circ S_{\bm{X}}^{(1)}, \norm{\cdot}_{\rm F}\big)\mathcal{N}\big(\epsilon / (L_{\ell}\sqrt{2}), \texttt{net} \circ S_{\bm{X}}^{(2)}, \norm{\cdot}_{\rm F}\big) \quad \text{(\texttt{softmax} is $1$-Lipschitz continuous)} \\
&\leq 
\exp \left( \dfrac{\norm{\bm{S}_{\bm{X}}^{(1)}}^2_{\rm F}}{\epsilon^2 / (2L^2_{\ell } )} R_{\texttt{NET}}+ \dfrac{\norm{\bm{S}_{\bm{X}}^{(2)}}^2_{\rm F}}{\epsilon^2 / (2L^2_{\ell } )} R_{\texttt{NET}}\right) \quad \text{(By Lemma~\ref{lemma:covering_number_resnet})}\\
&=\exp \left( \left( \norm{\bm{S}_{\bm{X}}^{(1)}}_{\rm F}^2 +\norm{\bm{S}_{\bm{X}}^{(2)}}_{\rm F}^2 \right) \dfrac{2R_{\texttt{NET}} L^2_{\ell }}{\epsilon^2} \right) \numberthis \label{eq:upper_bound_N}
\end{align*}
Eq. \eqref{eq:l_lipschizt} holds by applying Lemma~\ref{lemma:covering_number_contraction} 
and the fact that function $\phi_{y}(P_W) \triangleq -y \log(P_W) - (1-y) \log(1-P_W)$ is $L_\ell$-Lipschitz continuous. Specifically, $L_\ell = (1/\alpha)$ since for $\alpha \leq p_1, p_2 \leq 1-\alpha$,
\begin{align*}
\abs{\phi_y(p_1) - \phi_y(p_2)}
&= \abs{-y \log \dfrac{p_1}{p_2} - (1-y) \log \dfrac{1-p_1}{1-p_2}} \\
&\leq \abs{\log \dfrac{p_1}{p_2}} + \abs{\log \dfrac{1-p_1}{1-p_2}} \\
&\leq \abs{\dfrac{p_1 - p_2}{p_2}} + \abs{\dfrac{p_2 - p_1}{1-p_2}} \quad (\text{ by } \log (1+x) \leq x) \\
&\leq \dfrac{1}{\alpha} \abs{p_1 - p_2} 
= L_\ell  \abs{p_1 - p_2}.
\end{align*}
Similarly, Eq.~\eqref{eq:dot_lipschitz} holds by Lemma~\ref{lemma:covering_number_contraction} and the fact that
for $\; \forall  \bm{u}, \bm{v} \in \set{\bm{x} \mid \bm{x} \geq 0, \bm{1}^{\T}\bm{x} = 1}$,
function $\texttt{dot}([\bm{u}; \bm{v}])$ defined in \eqref{eq:dot_definition} is 2-Lipschitz continuous.
To see this, consider the following:
\[
\norm{\nabla \texttt{dot}([\bm{u}; \bm{v}])} = \norm{ \begin{bmatrix}
\bm{u} \\ \bm{v}
\end{bmatrix} } \leq 2.
\] 
Eq.~\eqref{eq:cover_cartesian_product} holds as a consequence of applying Lemma~\ref{lemma:cover_stack_set} to the set $\bm{f}' \circ S$ which is 
\begin{alignat*}{2}
&\bm{f}' \circ S &&= \set{[\bm{Z}_1, \bm{Z}_2] \in \mathbb{R}^{M \times 2K} \mid \bm{Z}_1 \in \bm{f} \circ S_{\bm{X}}^{(1)}, \bm{Z}_2 \in \bm{f} \circ S_{\bm{X}}^{(2)}}, \\
&\bm{S}^{(1)}_{\bm{X}} &&= [\bm{x}_{\omega_1^{1}}, \bm{x}_{\omega_2^{1}}, \ldots , \bm{x}_{\omega_M^{1}}] \in \mathbb{R}^{D \times M}, \\
&\bm{S}^{(2)}_{\bm{X}} &&= [\bm{x}_{\omega_1^{2}}, \bm{x}_{\omega_2^{2}}, \ldots , \bm{x}_{\omega_M^{2}}] \in \mathbb{R}^{D \times M}.
\end{alignat*}

Now we can proceed to bound $\mathcal{R}(\ell \circ \mathcal{P}_{\bm{X}} \circ S)$ with the help of Dudley's entropy integral.

Apply Lemma~\ref{lemma:chaining} onto the set $\ell  \circ \mathcal{P}_{\bm{X}} \circ S$ with $c=2\log(1/\alpha)\sqrt{M}$  (by \eqref{eq:upper_bound_loss}), for any integer $T > 0$, 
\begin{align*}
\mathcal{R}(\ell  \circ \mathcal{P}_{\bm{X}} \circ S) 
&\leq 
\dfrac{c2^{-T}}{\sqrt{M}} + \dfrac{6c}{M} \sum^{T}_{t=1} 2^{-t}\sqrt{\log (\mathcal{N}(c2^{-t}, A))} \quad \text{(By Lemma~\ref{lemma:chaining})}\\
&\leq \dfrac{c2^{-T}}{\sqrt{M}} + \dfrac{6c}{M} \sum^{T}_{t=1} 2^{-t} \sqrt{\left( \norm{\bm{S}_{\bm{X}}^{(1)}}_{\rm F}^2 +\norm{\bm{S}_{\bm{X}}^{(2)}}_{\rm F}^2 \right) \dfrac{2R_{\texttt{NET}} L^2_{\ell }}{c^2 2^{-2t}}} \quad \text{(By \eqref{eq:upper_bound_N})}\\
&\leq \dfrac{c2^{-T}}{\sqrt{M}} + \dfrac{6T}{M}  \sqrt{\left( \norm{\bm{S}_{\bm{X}}^{(1)}}_{\rm F}^2 +\norm{\bm{S}_{\bm{X}}^{(2)}}_{\rm F}^2 \right) 2R_{\texttt{NET}} L^2_{\ell }} \\
&= \dfrac{c2^{-T}}{\sqrt{M}} + \dfrac{6\sqrt{2}T}{M} \norm{\bm{S}_{\bm{X}}}_{\rm F} L_{\ell } \sqrt{R_{\texttt{NET}} } \\
&= 2\log(1/\alpha)2^{-T} + \dfrac{6\sqrt{2}T}{M} \norm{\bm{S}_{\bm{X}}}_{\rm F}L_{\ell } \sqrt{R_{\texttt{NET}} } \quad \text{(substitute $c$)}.
\end{align*}
Lastly, choosing $T = \lfloor \log_2 M \rfloor$ and substituting $L_\ell  = 1/\alpha$ concludes the proof,
\begin{align*}
\mathcal{R}(\ell  \circ \mathcal{P}_{\mathcal{F}, \bm{X}} \circ S) 
&\leq 2 \log (1/\alpha) 2^{-\lfloor \log_2 M \rfloor} + \dfrac{6\sqrt{2}{\lfloor \log_2 M \rfloor}}{M} \norm{\bm{S}_{\bm{X}}}_{\rm F} L_{\ell } \sqrt{R_{\texttt{NET}}} \\
&\leq \dfrac{4\log(1/\alpha)}{M} + \dfrac{6\sqrt{2} \log M}{\alpha M \log 2} \norm{\bm{S}_{\bm{X}}}_{\rm F}  \sqrt{R_{\texttt{NET}}}. 
\end{align*} 
This concludes the proof.
\end{proof}

\subsubsection{Some others lemmas}%
\label{ssub:some_others_lemmas}

\begin{lemma}
    \label{lemma:fro_to_kl}
    For $0 \leq \bm{P}, \bm{Q} \leq 1$ with the same size $N \times N$,
    \[
    \norm{\bm{P} - \bm{Q}}_{\rm F}^2 
    \leq 4N^2  D_{H}^2(\bm{P}, \bm{Q}) 
    \leq 4N^2 D_{\rm kl}(\bm{Q}||\bm{P}),
    \] 
    where 
    \begin{alignat*}{2}
    &D_{H}^2(\bm{P}, \bm{Q}) && \triangleq \dfrac{1}{N^2} \sum_{\omega \in [N] \times [N]}  d_H^{2}(P_{\omega}, Q_{\omega}) , \\
    &D_{\rm kl}(\bm{P} || \bm{Q}) && \triangleq \dfrac{1}{N^2} \sum_{\omega \in [N] \times [N]} d_{\rm kl}(P_{\omega} || Q_{\omega}) \numberthis \label{eq:dkl_matrix},
    \end{alignat*}
    and 
    $d_H(p,q)$, $d_{\rm kl}(p, q)$ are typical Hellinger distance and KL divergence for $0\leq p, q\leq 1$, resp, i.e.,
    \begin{alignat*}{2}
        &d_H^2(p, q) &&\triangleq \big(\sqrt{p} - \sqrt{q}\big)^2 + \big( \sqrt{1-p} - \sqrt{1-q} \big)^2, \\
        &d_{\rm kl}(p, q) && \triangleq p \log\left(\dfrac{p}{q}\right) + (1-p) \log \left(\dfrac{1-p}{1-q}\right). \numberthis \label{eq:dkl_element}
    \end{alignat*}
    
\end{lemma}
\begin{proof}
For the first inequality, we use the following result.
\begin{lemma}[Lemma 2, scalar version, \citep{davenport20141}]
    \label{lemma:hellinger_to_fr}
    Let $f: \mathbb{R} \rightarrow \mathbb{R}$ be any differential function, and s,t are two real numbers satisfying $\abs{s}, \abs{t} \leq \alpha$. Then
    \[
    d^2_H(f(s); f(t)) \geq \inf_{\abs{x} \leq \alpha} \dfrac{f'(x)^2}{8f(x)(1-f(x))} (s-t)^2.
    \] 
\end{lemma}
As a result, for $f(x) = x$, $0 \leq p,q \leq 1$,
    \[
    d^2_H(p, q) \geq \inf_{\abs{x} \leq 1} \dfrac{1}{8 x(1-x)} (p-q)^2 = \dfrac{1}{4}(p -q )^2.
    \] 
Summing across all elements of $\bm{P}, \bm{Q}$ leads to,
\[
    \norm{\bm{P} - \bm{Q}}_{\rm F}^2 \leq 4N^2 D^2_H (\bm{P}, \bm{Q}).
\] 
The second inequality is simply derived from Jensen's inequality and the fact that $1 - x \leq -\log x, \forall x > 0$.
\end{proof}

\begin{lemma}[Dudley entropy integral]
\label{lemma:chaining}
(\citep[Lemma 27.4]{shalev2014understanding})
Let $A \subset \mathbb{R}^{m}$, $c = \min_{\overline{\bm{a}} \in \mathbb{R}^{m}} \max_{\bm{a} \in A} \; \norm{\overline{\bm{a}} - \bm{a}}$. Then, for any integer $T > 0$, 
\[
R(A) \leq  \dfrac{c2^{-T}}{\sqrt{m}} + \dfrac{6c}{m} \sum^{T}_{k=1} 2^{-k}\sqrt{\log (N(c2^{-k}, A))}.
\] 
\end{lemma}

\begin{lemma}
    \label{lemma:covering_number_contraction}
    For each $i \in [m]$, let $\boldsymbol \phi_i: \mathbb{R}^{d_1} \rightarrow \mathbb{R}^{d_2}$ be a $\rho$-Lipschitz function; namely, for all  $ \bm{x}, \bm{y} \in \mathbb{R}^{d_1}$ we have $\norm{\boldsymbol \phi_i(\bm{x}) - \boldsymbol \phi_i(\bm{y})} \leq \rho \norm{\bm{x} - \bm{y}}$.
    For $\bm{A} \in \mathbb{R}^{d_1 \times m}$ let $\boldsymbol \Phi(\bm{A}) $ denote the matrix $[\boldsymbol \phi_1(\bm{a}_1), \ldots , \boldsymbol \phi_m(\bm{a}_m)]$.
    For $\mathcal{A} \subset \mathbb{R}^{d_1 \times m}$, let $\boldsymbol \Phi \circ \mathcal{A} = \set{\boldsymbol \Phi(\bm{A}) \mid \bm{A} \in \mathcal{A}} \subset \mathbb{R}^{d_2 \times m}$.
    Then,
    \begin{equation}
        \label{eq:covering_numer_contraction}
    \mathcal{N}(\rho r, \boldsymbol \Phi \circ \mathcal{A}, \norm{\cdot}_{\rm F}) \leq \mathcal{N}(r, \mathcal{A}, \norm{\cdot}_{\rm F}).
    \end{equation} 
\end{lemma}
\begin{remark}
Firstly, the choice of Frobenius norm in \eqref{eq:covering_numer_contraction} is tied to $\ell_2$-norm defining the $\rho$-Lipschitz function. Secondly, Lemma~\ref{lemma:covering_number_contraction} is a natural generalization of \citep[Lemma 27.3]{shalev2014understanding}.
\end{remark}
\begin{proof}
Define $\mathcal{A}'$ as one of the smallest $r$-cover of $\mathcal{A}$, i.e., for all $\bm{A} \in \mathcal{A}$ there exists $\bm{A}' \in \mathcal{A}'$ such that
$\norm{\bm{A} - \bm{A}'}_{\rm F} \leq r$. By definition, $\abs{\mathcal{A}'} = \mathcal{N}(r, \mathcal{A}, \norm{\cdot}_{\rm F})$.
Since
\[
\norm{\boldsymbol \Phi(\bm{A}) - \boldsymbol \Phi(\bm{A}')}_{\rm F}^2 
= \sum^{m}_{i=1} \norm{\boldsymbol \phi_i(\bm{a}_i) - \boldsymbol \phi_i(\bm{a}'_i)}^2
\leq \sum^{m}_{i=1} \rho^2 \norm{\bm{a}_i - \bm{a}'_i}^2
= \rho^2 \norm{\bm{A} - \bm{A}'}_{\rm F}^2
\leq \rho^2 r^2,
\] 
$\boldsymbol \Phi \circ \mathcal{A}'$ is a $\rho r$-cover of  $\boldsymbol \Phi \circ \mathcal{A}$.
That implies $\mathcal{N}(\rho r, \boldsymbol \Phi \circ \bm{A}, \norm{\cdot}_{\rm F}) \leq \abs{\boldsymbol \Phi \circ \mathcal{A}'} = \abs{\mathcal{A}'} = \mathcal{N}(r, \mathcal{A}, \norm{\cdot}_{\rm F})$.
\end{proof}
\begin{lemma}
    \label{lemma:cover_stack_set}
    Given two sets $\mathcal{A}_1, \mathcal{A}_2 \subseteq \mathbb{R}^{m \times d}$. Define
    $ \mathcal{A}_3 \triangleq \set{[\bm{A}_1, \bm{A}_2] \in \mathbb{R}^{m \times 2d} \mid \bm{A}_1 \in \mathcal{A}_1, \bm{A}_2 \in \mathcal{A}_2}$, then 
    \[
    \mathcal{N}(\epsilon/ \sqrt{2} , \mathcal{A}_3, \norm{\cdot}_{\rm F}) 
    \leq \mathcal{N}(\epsilon, \mathcal{A}_1, \norm{\cdot}_{\rm F}) \mathcal{N}(\epsilon, \mathcal{A}_2, \norm{\cdot}_{\rm F}).
    \] 
\end{lemma}
\begin{proof}
    Let $\overline{\mathcal{A}}_1, \overline{\mathcal{A}}_2$ be minimum $\epsilon$-cover sets of $\mathcal{A}_1, \mathcal{A}_2$, resp. By definition, for all $\bm{A}_1 \in \mathcal{A}_1, \bm{A}_2 \in \mathcal{A}_2$,
    \begin{align*}
        \exists \overline{\bm{A}}_1 \in \overline{\mathcal{A}}_1, \quad &\norm{\bm{A}_1 - \overline{\bm{A}}_1}_{\rm F}^2 \leq \epsilon^2, \text{ and} \\
        \exists \overline{\bm{A}}_2 \in \overline{\mathcal{A}}_2, \quad &\norm{\bm{A}_2 - \overline{\bm{A}}_2}_{\rm F}^2 \leq \epsilon^2
    \end{align*} 
    \[
    \Longrightarrow \quad \norm{[\bm{A}_1, \bm{A}_2] - [\overline{\bm{A}}_1, \overline{\bm{A}_2}]}^2_{\rm F} \leq 2 \epsilon^2.
    \] 
    Therefore, set $\overline{\mathcal{A}}_3 \triangleq \set{[\overline{\bm{A}}_1, \overline{\bm{A}}_2 ]\in \mathbb{R}^{m \times 2d} \mid \overline{\bm{A}}_1 \in \overline{\mathcal{A}}_1, \overline{\bm{A}}_2 \in \overline{\mathcal{A}}_2}$ is a $(\epsilon/\sqrt{2})$-cover set of $\mathcal{A}_3$, which leads to our conclusion,
    \[
    \mathcal{N}(\epsilon/\sqrt{2}, \mathcal{A}_3, \norm{\cdot}_{\rm F})
    \leq \abs{\overline{\mathcal{A}}_3}
    = \abs{\overline{\mathcal{A}}_1} \abs{\overline{\mathcal{A}}_2}
    = \mathcal{N}(\epsilon, \mathcal{A}_1, \norm{\cdot}_{\rm F})\mathcal{N}(\epsilon, \mathcal{A}_2, \norm{\cdot}_{\rm F}).
    \] 
\end{proof}
\begin{lemma}{Covering number of neural networks \citep[Theorem 3.3]{bartlett2017spectrally}}
    \label{lemma:covering_number_resnet}
    Let fixed nonlinearities $(\sigma_1, \ldots , \sigma_L)$ and reference matrices $(M_1, \ldots , M_L)$ be given, where $\sigma_i$ is $\rho_i$-Lipschitz and $\sigma_i(0) = 0$. Let spectral norm bounds $(s_1, \ldots , s_L)$, and matrix (2,1) norm bounds $b_1, \ldots , b_L$ be given. Let data matrix $\bm{X} \in \mathbb{R}^{n \times d}$ be given, where the $n$ rows correspond to data points. Let $\mathcal{H}_{\bm{X}}$ denote the family of matrices obtained by evaluating $\bm{X}$ with all choices of network $F_{\mathcal{A}}$:
    \[
    \mathcal{H}_{\bm{X}} \triangleq \set{F_{\mathcal{A}}(\bm{X}^{\T}) \mid \mathcal{A} = (\bm{A}_1, \ldots , \bm{A}_L), \norm{\bm{A}_i}_{\sigma} \leq s_i, \norm{\bm{A}_i^{\T} - \bm{I}}_{2,1} \leq b_i }, 
\] 
    where each matrix has dimension at most $D$ along each axis. Then for any  $\epsilon > 0$,
    \begin{align*}
    \log \mathcal{N} (\epsilon, \mathcal{H}_{\bm{X}}, \norm{\cdot}_{\rm F}) 
    &\leq \dfrac{\norm{\bm{X}}_{\rm F}^2 R_{\texttt{NET}}}{\epsilon^2},
    \end{align*} 
    where $R_{\texttt{NET}} = \log (2 D^2) \left( \prod_{j=1}^{L} s_j^2 \rho_j^2 \right) \left( \sum^{i}_{L} (b_i/s_i)^{2/3} \right)^{3}$ is a constant depending only on the neural network's properties. 
\end{lemma}

\section{Proof of Theorem~\ref{theorem:recovery_M_given_P}}%
\label{app:proof_theorem_recovery_M_given_P}

Let $\bm{M}^{\natural} = [\bm{f}^{\natural}(\bm{x}_1), \ldots , \bm{f}^{\natural}(\bm{x}_N)]$.
As $\bm{M}^{\natural}$ satisfies separability condition, we assume that $\bm{M}^{\natural } = [\bm{I}, \bm{V}^{\natural}]$, which is w.o.l.g.
Let $\bm{M}^{\star } = [\bm{f}^{\star }(\bm{x}_1), \ldots , \bm{f}^{\star }(\bm{x}_N)]$ denote the predicted cluster memberships, and also partition $\bm{M}^{\star }$ as $\bm{M}^{\star } = [\bm{U}^{\star }, \bm{V}^{\star }], \bm{U}^{\star } \in \mathbb{R}^{K \times K}$. Then, 
\begin{align*}
\norm{\bm{P}^{\star } - \bm{P}^{\natural}}_{\rm F}^2 
&= \norm{(\bm{M}^{\star })^{\T} \bm{M}^{\star } - (\bm{M}^{\natural})^{\T} \bm{M}^{\natural}}_{\rm F}^2 \\
&= \norm{
\begin{bmatrix}
    (\bm{U}^{\star })^{\T} \\ (\bm{V}^{\star })^{\T}
\end{bmatrix} 
\begin{bmatrix}
    \bm{U}^{\star } & \bm{V}^{\star }
\end{bmatrix} 
- \begin{bmatrix}
\bm{I}_K \\ (\bm{V}^{\natural})^{\T}
\end{bmatrix} 
\begin{bmatrix}
    \bm{I}_K & \bm{V}^{\natural}
\end{bmatrix}}_{\rm F}^2 \\
&= \norm{(\bm{U}^{\star })^{\T} \bm{U}^{\star } - \bm{I}_K}_{\rm F}^2 + 2 \norm{(\bm{U}^{\star })^{\T}\bm{V}^{\star } - \bm{V}^{\natural}}_{\rm F}^2 + \norm{(\bm{V}^{\star })^{\T}\bm{V}^{\star } - (\bm{V}^{\natural})^{\T} \bm{V}^{\natural}}_{\rm F}^2. \numberthis \label{eq:uvuvuv}
\end{align*}

\paragraph{Analysis of the first term in \eqref{eq:uvuvuv}.} 
Let $Z = \dfrac{1}{N^2} \norm{\bm{P}^{\star } - \bm{P}^{\natural}}_{\rm F}^2, 0\leq Z \leq 1$ be a random variable with a PDF $f(z)$. 
By Lemma~\ref{lemma:recovery_P},
\[
\pr{Z \leq \epsilon( M, \delta)^2}
    \geq 1-\delta.
\] 
We have
\begin{align*}
\mathop{\mathbb{E}}[Z] 
= \int^{1}_{0} z f(z)dz 
= \int^{\epsilon( M, \delta)^2}_{0} z f(z) dz + \int^{1}_{\epsilon( M, \delta )^2} z f(z) dz 
&\leq \int^{\epsilon( M, \delta)^2}_{0} \epsilon( M, \delta)^2 f(z)dz +\int^{1}_{\epsilon( M, \delta)^2}  f(z) dz \\
&\leq \epsilon( M, \delta)^2  + \delta.
\end{align*}
Recall $S=\set{(i_1, j_1, y_1), \ldots , (i_M, j_M, y_M)}$. According to Lemma~\ref{lemma:expectation_is_the_same},
\[
\mathop{\mathbb{E}}_{\bm{X}, S} \left[ (P^{\star }_{ij} - P^{\natural}_{ij})^2 \right] = \text{const}, \quad \; \forall  (i,j) \in [N]^2.
\] 
Therefore, for arbitrary $i, j \in [N]^2$,
\begin{align*}
\mathop{\mathbb{E}}_{\bm{X}, S} \left[ (P^{\star }_{ij} - P^{\natural}_{ij})^2 \right] 
&= \mathop{\mathbb{E}}_{\bm{X}, S} \left[ \dfrac{1}{N^2} \norm{\bm{P}^{\star } - \bm{P}^{\natural}}_{\rm F}^2 \right] \leq \epsilon( M, \delta)^2 + \delta.
\end{align*}
Meanwhile, by Markov inequality, for $\tau  > 0,$
\[
\pr{\left( P^{\star }_{ij} - P^{\natural}_{ij} \right)^2 
\leq \tau \mathop{\mathbb{E}}_{\bm{X}, S} \left[ \left( P^{\star }_{ij} - P^{\natural}_{ij} \right)^2 \right]} 
\geq 1 - \dfrac{1}{\tau }.
\] 
 
Setting $\tau = M^{1/4}, \delta = e^{-\sqrt{M}}$ gives
\begin{align*}
\left( P^{\star }_{ij} - P^{\natural}_{ij} \right)^2 
\leq M^{1/4} \mathop{\mathbb{E}}_{\bm{X}, S} \left[ \left( P^{\star }_{ij} - P^{\natural}_{ij} \right)^2 \right] 
\leq M^{1/4} \epsilon( M, e^{-\sqrt{M}})^2 + M^{1/4} e^{-\sqrt{M}}
\end{align*} 
holds with probability at least $1-1/M^{1/4}$.

By union bound, the following holds with probability at least $1-K^2/M^{1/4}$,
\begin{equation}
\label{eq:kxikabg}
\norm{(\bm{U}^{\star })^{\T}\bm{U}^{\star } - \bm{I}_K}_{\rm F}^2 
= \sum_{(i,j) \in [K]^2} (P^{\star }_{ij} - P^{\natural}_{ij})^2 
\leq K^2 \underbrace{(M^{1/4} \epsilon( M, e^{-\sqrt{M}})^2 + M^{1/4} e^{-\sqrt{M}})}_{\epsilon'(M)^2}
= K^2 \epsilon'(M)^2,
\end{equation}
where $\epsilon'(M)$ is defined in \eqref{eq:condition_for_M}.
Inequality \eqref{eq:kxikabg} implies 
\begin{numcases}{}
    1 - K\epsilon'(M) \leq \norm{\bm{u}^{\star }_k}^2 \leq 1, \quad \forall k \in [K] \label{eq:length_bound} \\
    0 \leq \langle \bm{u}^{\star }_k, \bm{u}^{\star }_\ell  \rangle \leq K \epsilon'(M), \quad \quad \; \forall k \neq \ell, k, \ell  \in [K].  \label{eq:angle_bound}
\end{numcases}

    For any permutation matrix $\boldsymbol \Pi \in \mathbb{R}^{K \times K}$, 
    \begin{align*}
    \norm{\boldsymbol \Pi \bm{U}^{\star } - \bm{I} }_{\rm F}^2
    = \sum^{K}_{k=1} \norm{\bm{u}^{\star }_k - \boldsymbol \pi_k}^2 
    &= K + \sum^{K}_{k=1} \norm{\bm{u}^{\star }_k}^2 - 2 \sum^{K}_{k=1} \langle \bm{u}^{\star }_k, \boldsymbol \pi_k \rangle \\
    & \leq 2K - 2 \sum^{K}_{k=1} \langle \bm{u}^{\star }_k, \boldsymbol \pi_k \rangle \quad \text{(Since $\bm{u}_k^{\star} \geq 0, \bm{1}^{\T}\bm{u}_k^{\star}=1$  and thus $\|\bm u_k^{\star }\|\leq 1$)}.
    \end{align*}

Minimizing over $\boldsymbol \Pi$ on both sides,
\begin{equation}
\label{eq:upper_X_I}
\min_{\boldsymbol \Pi} \norm{\boldsymbol \Pi \bm{U}^{\star } - \bm{I} }_{\rm F}^2
\leq 2K - 2 \max_{\boldsymbol \Pi} \sum^{K}_{k=1} \langle \bm{u}^{\star }_k, \boldsymbol \pi_k \rangle.
\end{equation}

Denote $k^{\star }_i := \argmax_{k} U^{\star }_{ki}$.
Invoking Lemma~\ref{lemma:different_argmax} where $K\epsilon'(M) < 0.381$ (by \eqref{eq:condition_for_M}) and $\bm{u}^{\star}_k$'s satisfy \eqref{eq:length_bound} and \eqref{eq:angle_bound}, it holds that
 $k^{\star }_1, \ldots , k^{\star }_K$ are all different integers, and we assume w.o.l.g. that $k^{\star }_1 = 1, \ldots , k^{\star }_K = K$, i.e., $U^{\star }_{ii} = \argmax_{k} U^{\star }_{ki}$.
With that in mind, we get a simple lower bound,
\begin{equation}
\label{eq:goal2}
\max_{\boldsymbol \Pi} \sum^{K}_{k=1} (\bm{u}_k^{\star })^{\T} \boldsymbol \pi_k
\geq \sum^{K}_{k=1} U^{\star }_{kk}.
\end{equation} 

Since $\norm{\bm{u}^{\star }_1}^2 \geq 1 - K\epsilon'(M)$ by \eqref{eq:length_bound},
\begin{align*}
(U^{\star }_{11})^2 &\geq 1- K \epsilon'(M) - \sum^{K}_{k=2} (U^{\star }_{k1})^2 
\geq 1- K \epsilon'(m) - \left(  \sum^{K}_{k=2} U^{\star }_{k1}\right)^2 
= 1- K \epsilon'(m) - (1-U^{\star }_{11})^2,
\end{align*}
where the second inequality holds since $U^{\star}_{k1} \geq 0, \; \forall k \in [K]$, and the last equality holds since $\bm{u}_1$ is a probability simplex vector.
As $U_{11}^{\star } = \max_{k} U_{k1}$, 
$U_{11}^{\star } \geq 1/K$.
Solving the above quadratic inequality with respect to $1/K \leq U^{\star }_{11} \leq 1$ and $K\epsilon'(M) \leq 0.381$ (by \eqref{eq:condition_for_M}) leads to
\[
U^{\star }_{11} \geq \dfrac{1 + \sqrt{1-2K\epsilon'(M)}}{2}
\geq 1 - K \epsilon'(M).
\] 
Applying the same argument for $U^{\star }_{22}, \ldots , U^{\star }_{KK}$, 
\[
U^{\star }_{kk} \geq 1 -  K \epsilon'(M), \quad k \in [K].
\] 
Therefore,
\begin{equation}
\label{eq:bound_first_partition}
\min_{\boldsymbol \Pi} \norm{\boldsymbol \Pi \bm{U}^{\star } - \bm{I}}_{\rm F}^2 
\leq 2K - 2 \left( K -  K^2 \epsilon'(M)\right) 
= 2 K^2 \epsilon'(M).
\end{equation} 
This implies that $\bm{U}^{\star }$ is close to an identity matrix up to some permutation.

\paragraph{Analysis of the second term in \eqref{eq:uvuvuv}.} 

For a fixed $0 < \delta < 1$, suppose that the following event happens
\begin{equation}
\label{eq:another_event}
\dfrac{1}{N^2} \norm{\bm{P}^{\star } - \bm{P}^{\natural}}_{\rm F}^2 \leq \epsilon(M, \delta)^2.
\end{equation}
Recall that $\bm{M}^{\star } = [\bm{U}^{\star }, \bm{V}^{\star }], \Msharp= [\bm{I}_K, \bm{V}^{\natural}]$, then
\[
\norm{(\bm{U}^{\star })^{\T} \bm{V}^{\star } - \bm{V}^{\natural}}_{\rm F} \leq \dfrac{N \epsilon( M, \delta )}{\sqrt{2}}.
\] 
Denote $\boldsymbol \Pi^{\star }$ as the optimal permutation matrix in LHS of \eqref{eq:bound_first_partition}. We have
\begin{align*}
    \norm{(\bm{U}^{\star })^{\T} \bm{V}^{\star } - \bm{V}^{\natural} - ((\bm{U}^{\star })^{\T} (\boldsymbol \Pi^{\star })^{\T} - \bm{I})  \bm{V}^{\natural}}_{\rm F}
    & \leq \norm{(\bm{U}^{\star })^{\T}\bm{V}^{\star } - \bm{V}^{\natural}}_{\rm F} + \norm{((\bm{U}^{\star })^{\T}(\boldsymbol \Pi^{\star })^{\T} -\bm{I})\bm{V}^{\natural}}_{\rm F} \\
    &\leq \dfrac{N \epsilon( M, \delta )}{\sqrt{2}} + \sigma_{\max}(\bm{V}^{\natural}) \norm{\boldsymbol \Pi^{\star } \bm{U}^{\star } - \bm{I}}_{\rm F} \\
    &\leq \dfrac{N \epsilon( M, \delta )}{\sqrt{2}} + \sigma_{\max}(\bm{V}^{\natural}) K \sqrt{2 \epsilon'(M)} \quad \text{(thanks to \eqref{eq:bound_first_partition})} \numberthis \label{eq:eps_21}.
\end{align*}
On the other hand,
\begin{align*}
    \norm{(\bm{U}^{\star })^{\T}\bm{V}^{\star } - \bm{V}^{\natural} - ((\bm{U}^{\star })^{\T}(\boldsymbol\Pi^{\star })^{\T} - \bm{I}) \bm{V}^{\natural}}_{\rm F}
    &= \norm{(\bm{U}^{\star })^{\T}(\bm{V}^{\star } - (\boldsymbol \Pi^{\star })^{\T} \bm{V}^{\natural})}_{\rm F} \\
    & \geq \sigma_{\min}(\bm{U}^{\star }) \norm{\bm{V}^{\star } - (\boldsymbol \Pi^{\star })^{\T} \bm{V}^{\natural}}_{\rm F} \\
    & = \sigma_{\min}(\bm{U}^{\star }) \norm{\boldsymbol \Pi^{\star } \bm{V}^{\star } - \bm{V}^{\natural}}_{\rm F} \numberthis \label{eq:sigma_min_fro}.
\end{align*}
Using Lemma~\ref{lemma:perturbation_theory}, we can upper bound $\sigma_{\min}(\bm{U}^{\star })$ as
\[
\abs{\sigma_{\min}(\bm{I} + \boldsymbol \Pi^{\star } \bm{U}^{\star } - \bm{I})
- \sigma_{\min}(\bm{I})} 
\leq \norm{\boldsymbol \Pi^{\star } \bm{U}^{\star } - \bm{I}}_{\rm F}
\leq K\sqrt{2 \epsilon'(M)} \quad \text{(thanks to \eqref{eq:bound_first_partition})}
\] 
or,
\begin{equation}
\label{eq:bound_sigma_min}
\sigma_{\min}(\bm{U}^{\star }) 
=\sigma_{\min}(\boldsymbol \Pi^{\star } \bm{U}^{\star })
\geq 1 - K\sqrt{2  \epsilon'(M)}.
\end{equation} 
Combine \eqref{eq:eps_21}, \eqref{eq:sigma_min_fro}, \eqref{eq:bound_sigma_min},
\begin{align*}
\norm{\boldsymbol \Pi^{\star } \bm{V}^{\star } - \bm{V}^{\natural}}_{\rm F}^2
&\leq \left(  \dfrac{ N\epsilon(M, \delta) + 2\sigma_{\max}(\bm{V}^{\natural}) K \sqrt{\epsilon'(M)}}{\sqrt{2} - 2K\sqrt{\epsilon'(M)}} \right)^2\\
&\leq  2 \left(  N\epsilon(M, \delta) + 2\sigma_{\max}(\bm{V}^{\natural})K\sqrt{\epsilon'(M)} \right)^2  \quad \text{(Since $8K^2 \epsilon'(M) \leq 1$ by \eqref{eq:condition_for_M})}\\
&\leq  4 N^2 \epsilon(M, \delta)^2  + 16 \sigma^2_{\max}(\bm{V}^{\natural}) K^2 \epsilon'(M) \quad \text{(Cauchy-Schwarz inequality)}.
\numberthis \label{eq:bound_second_partition} 
\end{align*} 

Lastly, combine \eqref{eq:bound_first_partition} and \eqref{eq:bound_second_partition}, we finish our proof,
\begin{align*}
\norm{\boldsymbol \Pi^{\star } \bm{M}^{\star } - \Msharp}_{\rm F}^2 
&=\norm{\boldsymbol \Pi^{\star } \bm{U} - \bm{I}}_{\rm F}^2 + \norm{\boldsymbol \Pi^{\star } \bm{V} - \bm{V}^{\natural}}_{\rm F}^2 \\
&\leq 2K^2 \epsilon'(M) + 4N^2 \epsilon(M, \delta)^2 + 16 \sigma_{\max}^2(\bm{V}^{\natural})K^2 \epsilon'(M) \\
&= 4N^2 \epsilon(M, \delta)^2 + 2K^2 (1 + 8 \sigma_{\max}^2(\bm{V}^{\natural})) \epsilon'(M).
\end{align*} 
Note that the whole computation is derived based on the realizations of 2 independent events in \eqref{eq:kxikabg}, \eqref{eq:another_event} with probability of happening at least $1-K^2/M^{1/4}$ and $1-\delta$, resp,
\begin{equation*}
    \dfrac{1}{N}\norm{\boldsymbol \Pi^{\star } \bm{M}^{\star } - \bm{M}^{\natural}}_{\rm F}^2 
    \leq 4N \epsilon(M, \delta)^2 + \dfrac{2}{N}K^2(1+8\sigma_{\max}^2(\bm{V}^{\natural})) \epsilon'(M).
\end{equation*}
holds with probability at least $1-(\delta + K^2/M^{1/4})$.

\subsection{Supporting Lemmas for proof of Theorem~\ref{theorem:recovery_M_given_P}}%
\label{sub:supporting_lemmas}

\begin{lemma}[\cite{weyl1912asymptotische}]
    \label{lemma:perturbation_theory}
    Let $\bm{X}, \boldsymbol \Delta \in \mathbb{R}^{m \times n}$, 
    \[
    \abs{\sigma_{i}(\bm{X}+ \boldsymbol \Delta) -\sigma_{i}(\boldsymbol \Delta)} \leq \norm{\boldsymbol \Delta}_2 \quad (\; \leq \norm{\boldsymbol \Delta}_{\rm F}), \quad 1 \leq i \leq \min (m, n).
    \] 
\end{lemma}
\begin{lemma}
    \label{lemma:different_argmax}
    Define $A_{\epsilon} = \set{\bm{x} \in \mathbb{R}^{K} \mid \bm{1}^{\T} \bm{x} = 1, \bm{x} \geq 0, \norm{\bm{x}}_2^2 \geq 1 - \epsilon}$. If $\epsilon \leq 0.381$, then
    \[
     \max_{\bm{x}, \bm{y} \in A_{\epsilon}} \bm{x}^{\T}\bm{y} \leq \epsilon 
     \quad \Rightarrow \quad 
     \argmax_{k} x_{k} \neq \argmax_{k} y_{k}.
\]
\end{lemma}
\begin{proof}
    We use contradiction to prove the claim. Suppose that $\argmax_{k} x_k = \argmin_{k} y_k = \widetilde{k}$.
    Without loss of generality, let us assume $\widetilde{k}=1$. Using these assumptions, we will show that
    $ \min_{\bm{x}, \bm{y} \in A_{\epsilon}} \; \bm{x}^{\T}\bm{y} > \epsilon $. 
    Indeed, 
    \begin{align*}
    1- \epsilon \leq \norm{\bm{x}}^2 
    \leq x_1 \left(\sum^{K}_{k=1} x_k \right)
    = x_1.
    \end{align*}
    Similarly, we obtain $y_1 \geq 1-\epsilon$, which leads to
    \[
    \bm{x}^{\T}\bm{y} \geq (1-\epsilon)^2 > \epsilon,
    \] 
    where the second inequality holds because $\epsilon \leq 0.381$.
\end{proof}

\begin{lemma}
    \label{lemma:expectation_is_the_same}
     Assume that $\omega=(\omega^{1},\omega^{2})$  is a uniform RV over $[N]\times [N]$ where $\omega^{1} \in [N], \omega^{2} \in [N]$, and $\bm{x}_1, \ldots , \bm{x}_N$ are i.i.d.
     Define the following mappings:
     \begin{alignat*}{2}
     &\texttt{pick}(\bm{x}_1, \ldots , \bm{x}_N, \omega) &&\triangleq (\bm{x}_{\omega^{1}}, \bm{x}_{\omega^{2}}), \\
     &\texttt{gen}(\omega_1, \ldots , \omega_M, y_1, \ldots , y_M, \bm{x}_1, \ldots , \bm{x}_N) &&\triangleq \set{\boldsymbol \psi_1, \ldots , \boldsymbol \psi_M}, \quad
     \boldsymbol \psi_i \triangleq (\texttt{pick}(\bm{x}_1, \ldots , \bm{x}_N, \omega_i), y_i),
     \end{alignat*}
     where
     \begin{alignat*}{2}
     &\omega_i = (k, \ell) \in [N]^2,\quad && 1\leq i\leq M, \\
     &y_i \in \set{0, 1}, && 1\leq i\leq M, \\
     &\bm{x}_i \in \mathcal{X}, && 1\leq i\leq N.
     \end{alignat*}
    For a real-valued function $h(\bm{u}, \bm{v}, \texttt{gen}(\omega_1, \ldots , \omega_M, y_1, \ldots , y_M, \bm{x}_1, \ldots , \bm{x}_N))$ with $ \bm{u}, \bm{v} \in \mathcal{X}$, we have:
 $\; \forall  y_1, \ldots , y_M$, and $\; \forall i, j, i', j' \in [N]$,
    \begin{multline}
        \label{eq:expectation_elementwise}
    \mathop{\mathbb{E}}_{\substack{\omega_i, 1\leq i\leq M, \\ \bm{x}_i, 1\leq i\leq N}}[h(\bm{x}_i, \bm{x}_j, \texttt{gen}(\omega_1, \ldots , \omega_M, y_1, \ldots , y_M, \bm{x}_1, \ldots , \bm{x}_N))]
    = \\
    \mathop{\mathbb{E}}_{\substack{\omega_i, 1\leq i\leq M, \\ \bm{x}_i, 1\leq i\leq N}}[h(\bm{x}_{i'}, \bm{x}_{j'}, \texttt{gen}(\omega_1, \ldots , \omega_M, y_1, \ldots , y_M, \bm{x}_1, \ldots , \bm{x}_N))].
    \end{multline} 
\end{lemma}

\begin{proof}
     We have the following property on the mapping $\texttt{gen}$: 
     $\; \forall  i\neq j, i, j \in [N]$,
     \begin{multline}
         \label{eq:pairwise_data}
     \texttt{gen}(\omega_1, \ldots , \omega_M, y_1, \ldots , y_M, \ldots, \bm{x}_i, \ldots ,\bm{x}_j,\ldots, \bm{x}_N)
     =  \\
     \texttt{gen}(s_{ij}(\omega_1), \ldots , s_{ij}(\omega_M), y_1, \ldots , y_M, \ldots, \bm{x}_j, \ldots ,\bm{x}_i,\ldots, \bm{x}_N),
     \end{multline} 
     where $s_{ij}: [N]^2 \to [N]^2$,
     \[
     s_{ij}(\omega) = s_{ij}((k, \ell)) \triangleq \begin{cases}
         (k, \ell) \quad & \text{ if } k \neq i \text{ and } \ell \neq j \\
         (j,\ell)  \quad & \text{ if } k = i \text{ and } \ell \neq j \\
         (k,i)  \quad & \text{ if } k \neq i \text{ and }\ell = j  \\
         (j,i)  \quad & \text{ if } k = i \text{ and } \ell = j 
     \end{cases}
     \] 
     To see this, let 
     \begin{align*}
         \mathcal{G}_1 &= \texttt{gen}(\omega_1, \ldots , \omega_M, y_1, \ldots , y_M, \ldots , \bm{x}_i, \ldots , \bm{x}_j, \ldots , \bm{x}_N), \\
         \mathcal{G}_2 &= \texttt{gen}(s_{ij}(\omega_1), \ldots , s_{ij}(\omega_M), y_1, \ldots , y_M, \ldots , \bm{x}_j, \ldots , \bm{x}_i, \ldots , \bm{x}_N).
     \end{align*}
 Pick $\boldsymbol \psi = (\texttt{pick}(\bm{x}_1, \ldots , \bm{x}_i, \ldots , \bm{x}_j, \ldots , \bm{x}_N, \omega), y) \in \mathcal{G}_1$. Assume $\omega = (k, \ell)$. Consider the following cases,
 \begin{itemize}
     \item If $k \neq i$ and $\ell \neq j$, then $s_{ij}(\omega) = \omega$. So 
         \[
         \boldsymbol \psi = (\texttt{pick}(\bm{x}_1, \ldots , \bm{x}_j, \ldots , \bm{x}_i, \ldots , \bm{x}_N, s_{ij}(\omega)), y) \in \mathcal{G}_2.
         \] 
     \item If $k = i$ and $\ell \neq j$, then $s_{ij}(\omega) = s_{ij}((i, \ell)) = (j, \ell)$. So 
         \begin{align*}
             \boldsymbol \psi 
             &= (\texttt{pick}(\bm{x}_1, \ldots , \bm{x}_i, \ldots , \bm{x}_j, \ldots , \bm{x}_N, (i, \ell)), y)  \\
             &= (\bm{x}_i, \bm{x}_\ell, y) \\
             &= (\texttt{pick}(\bm{x}_1, \ldots , \bm{x}_j, \ldots , \bm{x}_i, \ldots , \bm{x}_N, (j, \ell)), y) \\
             &= (\texttt{pick}(\bm{x}_1, \ldots , \bm{x}_j, \ldots , \bm{x}_i, \ldots , \bm{x}_N, s_{ij}(\omega)), y) \in \mathcal{G}_2 \quad \text{(by definition of the mapping \texttt{gen})}.
         \end{align*} 
     \item If $k \neq i$ and $\ell = j$, then $s_{ij}(\omega) = s_{ij}((k, j)) = (k, i)$. So 
         \begin{align*}
             \boldsymbol \psi 
             &= (\texttt{pick}(\bm{x}_1, \ldots , \bm{x}_i, \ldots , \bm{x}_j, \ldots , \bm{x}_N, (k, j)), y)  \\
             &= (\bm{x}_k, \bm{x}_j, y) \\
             &= (\texttt{pick}(\bm{x}_1, \ldots , \bm{x}_j, \ldots , \bm{x}_i, \ldots , \bm{x}_N, (k, i)), y) \\
             &= (\texttt{pick}(\bm{x}_1, \ldots , \bm{x}_j, \ldots , \bm{x}_i, \ldots , \bm{x}_N, s_{ij}(\omega)), y) \in \mathcal{G}_2 \quad \text{(by definition of the mapping \texttt{gen})}.
         \end{align*} 
     \item If $k = i$ and $\ell = j$, then $s_{ij}(\omega) = s_{ij}((i, j)) = (j, i)$. So 
         \begin{align*}
             \boldsymbol \psi 
             &= (\texttt{pick}(\bm{x}_1, \ldots , \bm{x}_i, \ldots , \bm{x}_j, \ldots , \bm{x}_N, (i, j)), y)  \\
             &= (\bm{x}_i, \bm{x}_j, y) \\
             &= (\texttt{pick}(\bm{x}_1, \ldots , \bm{x}_j, \ldots , \bm{x}_i, \ldots , \bm{x}_N, (j, i)), y) \\
             &= (\texttt{pick}(\bm{x}_1, \ldots , \bm{x}_j, \ldots , \bm{x}_i, \ldots , \bm{x}_N, s_{ij}(\omega)), y) \in \mathcal{G}_2 \quad \text{(by definition of the mapping \texttt{gen})}.
         \end{align*} 
 \end{itemize}
Since $\psi \in \mathcal{G}_1$ is arbitrary, $\mathcal{G}_1 \subseteq \mathcal{G}_2$. Moreover, $\abs{\mathcal{G}_1} = \abs{\mathcal{G}_2}$, so $\mathcal{G}_1 = \mathcal{G}_2$, and hence \eqref{eq:pairwise_data} holds.

Notice that the mapping $s_{ij}(\omega)$ is surjective and onto $[N]^2$, we proceed with the following
     \begin{align*}
     & \mathop{\mathbb{E}}_{\substack{\bm{x}_i, \bm{x}_{i'},\\ \bm{x}_{k, k=[N]\setminus \set{i,i'}}, \\ \omega_1, \ldots , \omega_M}} \left[ h(\bm{x}_i, \bm{x}_j, \texttt{gen}(\omega_1, \ldots , \omega_M, y_1, ..., y_M, \ldots  ,\bm{x}_i, \ldots , \bm{x}_{i'}, \ldots )) \right]  \\
     &\quad = \mathop{\mathbb{E}}_{\substack{\bm{x}_i, \bm{x}_{i'},\\ \bm{x}_{k, k=[N]\setminus \set{i,i'}}, \\ \omega_1, \ldots , \omega_M}} \left[ h(\bm{x}_i, \bm{x}_j, \texttt{gen}(s_{ii'}(\omega_1), \ldots , s_{ii'}(\omega_M), y_1, \ldots , y_M, \ldots , \bm{x}_{i'}, \ldots ,\bm{x}_i, \ldots) \right]   \\
     &\quad = \mathop{\mathbb{E}}_{\substack{\bm{u}, \bm{v},\\ \bm{x}_{k, k=[N]\setminus \set{i,i'}}, \\ \omega_1, \ldots , \omega_M}} \left[ h(\bm{u}, \bm{x}_j, \texttt{gen}(s_{ii'}(\omega_1), \ldots , s_{ii'}(\omega_M), y_1, \ldots , y_M, \ldots , \bm{v}, \ldots , \bm{u}, \ldots)) \right]   \\
     &\quad = \mathbb{E}_{\substack{\bm{x}_{i'}, \bm{x}_i,\\ \bm{x}_{k, k=[N]\setminus \set{i,i'}}, \\ \omega_1, \ldots , \omega_M}} \left[ h(\bm{x}_{i'}, \bm{x}_j, \texttt{gen}(s_{ii'}(\omega_1), \ldots , s_{ii'}(\omega_M), y_1, \ldots , y_M, \ldots , \bm{x}_i, \ldots , \bm{x}_{i'}, \ldots)) \right].
     \end{align*}
     The first equality holds by property \eqref{eq:pairwise_data}, the second and the third equality hold as we just rename some random variables under the expectation.
     Proceed with the same argument to replace $j$ with  $j'$, we obtain 
     \begin{multline}
         \label{eq:proof_elementwise}
     \mathbb{E}_{\substack{\bm{x}_1, \ldots ,  \bm{x}_N, \\ \omega_1, \ldots , \omega_M}} \left[ h(\bm{x}_{i'}, \bm{x}_j, \texttt{gen}(s_{ii'}(\omega_1), \ldots , s_{ii'}(\omega_M), y_1, \ldots , y_M, \bm{x}_1, \ldots ,\bm{x}_N)) \right]  \\
     =  \mathbb{E}_{\substack{\bm{x}_1, \ldots ,  \bm{x}_N, \\ \omega_1, \ldots , \omega_M}} \left[ h(\bm{x}_{i'}, \bm{x}_{j'}, \texttt{gen}(s_{jj'}(s_{ii'}(\omega_1)), \ldots , s_{jj'}(s_{ii'}(\omega_M)), y_1, \ldots , y_M, \bm{x}_1, \ldots , \bm{x}_{N})) \right].
     \end{multline}
     Since $s_{ii'}(\omega)$ is surjective and onto itself, $s_{ii'}(\omega)$ has the same probability distribution as $\omega$. So $s_{jj'}(s_{ii'}(\omega))$ also define a RV with the same PDF as $\omega$,
     and hence \eqref{eq:proof_elementwise} implies \eqref{eq:expectation_elementwise}.

\end{proof}

\section{Proof of Theorem~\ref{theorem:generalization}}%
\label{app:proof_theorem_generalization}
We start with the following key lemma:
\begin{lemma}
    \label{lemma:generalization}
    With probability at least $1-\delta$ over $\bm{x}_1, \ldots , \bm{x}_N$,
    \[
    \mathop{\mathbb{E}}_{\bm{x} \sim \mathcal{P}_{\mathcal{X}}} \left[ \norm{\boldsymbol \Pi^{\star } \bm{f}^{\star }(\bm{x}) - \bm{f}^{\natural}(\bm{x})}^2 \right]
    \leq \dfrac{1}{N} \norm{\boldsymbol \Pi^{\star } \bm{M}^{\star } - \bm{M}^{\natural}}_{\rm F}^2 
    + \dfrac{8}{N} + \dfrac{12\sqrt{2 R_{\texttt{NET}}} \norm{\bm{X}}_{\rm F} \log N}{N \log 2} + 8 \sqrt{\dfrac{2 \log(4\delta)}{N}},
    \] 
where $\boldsymbol \Pi^{\star }$ is the optimal permutation matrix in LHS of \eqref{eq:bound_first_partition} in the proof of Theorem~\ref{theorem:recovery_M_given_P}.

\end{lemma}
\begin{proof}
Let us define the following:
\begin{alignat*}{2}
&L_{\bm{X}}(\bm{f}) &&\triangleq \dfrac{1}{N} \sum^{N}_{n=1} {\rm SE} (\bm{f}, (\bm{x}_n, \bm{m}^{\natural}_n)), \\
&{\rm SE}(\bm{f}, (\bm{x}, \bm{m})) &&\triangleq \norm{\boldsymbol \Pi^{\star }\bm{f}(\bm{x}) - \bm{m}}^2.
\end{alignat*} 
Note that $L_{\bm{X}}(\bm{f})$ can be viewed as the empirical loss in statistical learning where a dataset of $N$ i.i.d samples $(\bm{x}_1, \bm{m}_1^{\natural}), \ldots , (\bm{x}_N, \bm{m}^{\natural}_N)$ are given. 
Denote $\mathcal{P}_{\mathcal{X}, \bm{f}^{\natural}}$ as the PDF of the joint distribution of $(\bm{x}_n, \bm{m}^{\natural}_{n})$, and define
\[
L_{\mathcal{P}_{\mathcal{X}, \bm{f}^{\natural}}}(\bm{f}) \triangleq \mathop{\mathbb{E}}_{(\bm{x}, \bm{m}^{\natural}) \sim \mathcal{P}_{\mathcal{X}, \bm{f}^{\natural}}} \left[ {\rm SE}(\bm{f}, (\bm{x}, \bm{m}^{\natural})) \right],
\] 
which is the expected loss.
Then it is readily to apply a standard generalization bound from \citep[Theorem 26.5]{shalev2014understanding} using the fact that ${\rm SE}(\bm{f}, (\bm{x}, \bm{m})) \leq 2$. That is, we have that
\begin{equation*}
L_{\mathcal{P}_{\mathcal{X}, \bm{f}^{\natural}}}(\bm{f}^{\star })
\leq L_{\bm{X}}(\bm{f}^{\star }) + 2\mathcal{R}({\rm SE} \circ \mathcal{F} \circ \bm{X}) + 8 \sqrt{\dfrac{2 \log(4/\delta)}{N}},
\end{equation*}
holds with probability of at least $1-\delta$.
The first term is 
\[
L_{\bm{X}}(\bm{f}^{\star }) = \dfrac{1}{N} \norm{\boldsymbol \Pi^{\star } \bm{M}^{\star } - \bm{M}^{\natural}}_{\rm F}^2.
\] 

The second term is bounded by following the proof of Lemma~\ref{lemma:rademacher},
\begin{align*}
\mathcal{N} \left(  \epsilon, {\rm SE} \circ \mathcal{F} \circ \bm{X}, \norm{\cdot} \right)
&\leq \mathcal{N} \left(  \dfrac{\epsilon}{2\sqrt{2}}, \mathcal{F} \circ \bm{X}, \norm{\cdot}_{\rm F}\right) \numberthis \label{eq:ikqix} \\
&\leq \exp \left( \dfrac{8R_{\texttt{NET}} \norm{\bm{X}}_{\rm F}^2}{\epsilon^2}  \right) \quad \text{(By Lemma~\ref{lemma:covering_number_resnet})}.
\end{align*}
Inequality \eqref{eq:ikqix} holds by applying Lemma~\ref{lemma:covering_number_contraction} and the fact that function $\phi_{\bm{m}}(\bm{x}) \triangleq \norm{\bm{x} - \bm{m}}^2$ is $2\sqrt{2}$-Lipschitz continuous on the domain $\boldsymbol \Delta = \set{\bm{x} \in \mathbb{R}^{K} \mid \bm{x}\geq 0, \bm{1}^{\T}\bm{x}= 1}$. This is because of the following fact:
\[
\norm{\nabla \phi_{\bm{m}}(\bm{x})} = 2 \norm{\bm{x} - \bm{m}} \leq 2 \sqrt{2}.
\] 
We proceed by deriving a bound on $\mathcal{R}({\rm SE} \circ \mathcal{F} \circ \bm{X})$ with the help of Dudley's entropy integral technique. In particular, applying Lemma~\ref{lemma:chaining} on the set ${\rm SE} \circ \mathcal{F} \circ \bm{X}$, the following holds: For any integer $T>0$,
\begin{align*}
\mathcal{R}({\rm SE} \circ \mathcal{F} \circ \bm{X}) 
&\leq 
\dfrac{c2^{-T}}{\sqrt{N}} + \dfrac{6c}{N} \sum^{T}_{t=1} 2^{-t}\sqrt{\log (\mathcal{N}(c2^{-t}, {\rm SE} \circ \mathcal{F} \circ \bm{X}))} \\
&\leq \dfrac{c2^{-T}}{\sqrt{N}} + \dfrac{6c}{N} \sum^{T}_{t=1} 2^{-t} \sqrt{\dfrac{8R_{\texttt{NET}} \norm{\bm{X}}_{\rm F}^2}{c^2 2^{-2t}} } \\
&= \dfrac{c2^{-T}}{\sqrt{N}} + \dfrac{12T\sqrt{2R_{\texttt{NET}}}\norm{\bm{X}}_{\rm F}}{N}.
\end{align*}
Substituting $c=2$ as the upper bound of the loss ${\rm SE}$ and setting $T=\lfloor 0.5\log_2(N) \rfloor$ gives 
\[
\mathcal{R}({\rm SE} \circ \mathcal{F} \circ \bm{X}) 
\leq \dfrac{4}{N} + \dfrac{6\sqrt{2 R_{\texttt{NET}}} \norm{\bm{X}}_{\rm F} \log N}{N \log 2}.
\] 
Therefore,
\begin{multline*}
\mathop{\mathbb{E}}_{(\bm{x}, \bm{m}^{\natural}) \sim \mathcal{P}_{\mathcal{X}}, \bm{f}^{\natural}} \left[ \norm{ \boldsymbol \Pi^{\star }\bm{f}^{\star }(\bm{x}) - \bm{m}^{\natural}}^2 \right] 
\leq \dfrac{1}{N} \norm{\boldsymbol \Pi^{\star } \bm{M}^{\star } - \bm{M}^{\natural}}_{\rm F}^2
+ \dfrac{8}{N} + \dfrac{12\sqrt{2 R_{\texttt{NET}}} \norm{\bm{X}}_{\rm F} \log N}{N \log 2}
+ 8 \sqrt{\dfrac{2 \log (4/\delta)}{N}}
\end{multline*} 
holds with probability of at least $1-\delta$.
\end{proof}

By Theorem~\ref{theorem:recovery_M_given_P}, 
\[
L_{\bm{X}}(\bm{f}^{\star }) 
= \dfrac{1}{N} \min_{\boldsymbol \Pi } \sum^{N}_{n=1} \norm{\boldsymbol \Pi \bm{M}^{\star } - \bm{M}^{\natural}}_{\rm F}^2 
\leq 
4N \epsilon(M, \delta)^2 + \dfrac{2}{N}K^2(1+8\sigma_{\max}^2(\bm{V}^{\natural})) \epsilon'(M)
\] 
holds with probability of at least $1-\delta - K^2/M^{0.25}$. Invoking Lemma~\ref{lemma:generalization} gives the conclusion,
\begin{multline*}
\mathop{\mathbb{E}}_{(\bm{x}, \bm{m}^{\natural}) \sim \mathcal{P}_{\mathcal{X}}, \bm{f}^{\natural}} \left[ \norm{ \boldsymbol \Pi^{\star }\bm{f}^{\star }(\bm{x}) - \bm{m}^{\natural}}^2 \right] 
\leq 
4N \epsilon(M, \delta)^2 + \dfrac{2}{N}K^2(1+8\sigma_{\max}^2(\bm{V}^{\natural})) \epsilon'(M) \\
+ \dfrac{8}{N} + \dfrac{12\sqrt{2 R_{\texttt{NET}}} \norm{\bm{X}}_{\rm F} \log N}{N \log 2}
+ 8 \sqrt{\dfrac{2 \log (4/\delta)}{N}},
\end{multline*} 
hold with probability of at least $1-\delta - K^2/M^{0.25}$.

\section{Proof of Theorem~\ref{theorem:identifiability_limit}}%
\label{app:proof_theorem_identifiability_limit}

As $\max \big(\log(1/\alpha), \log (M)\sqrt{R_{\texttt{net}}}/\alpha \big)/\sqrt{M} \rightarrow 0$, $\norm{\bm{P}^{\star } - \bm{P}^{\natural}}_{\rm F}^2 \to 0$ accords to Lemma~\ref{lemma:recovery_P}, which means
\[
(\bm{M}^{\star })^{\T}\bm{M}^{\star } = \bm{P}^{\natural}.
\]
This gives a guarantee that there exists permutation matrix $\boldsymbol \Pi$ such that $\boldsymbol \Pi\bm{M}^{^{\star }} = \bm{M}^{\natural}$ by invoking \citep[Theorem 4]{huang2014non}. Lastly, applying Lemma~\ref{lemma:generalization} concludes the proof.

\section{Proof of Theorem~\ref{theorem:identifiability_limit_reg}}%
\label{app:proof_of_identifiability_limit_reg}

Invoking Lemma~\ref{lemma:recovery_P} and considering $\max \big(\log(1/\alpha), \log (M)\sqrt{R_{\texttt{net}}}/\alpha \big)/\sqrt{M} \rightarrow 0$ gives
\[
\bm{P}^{\star } = (\bm{M}^{\star })^{\T}  (\bm{A}^{\star })^{\T} \bm{A}^{\star } \bm{M}^{\star } = \bm{P}^{\natural}
\] 
at the limit.
As assumed in Theorem~\ref{theorem:identifiability_limit_reg}, $\rank((\bm{M}^{\natural })^{\T}  (\bm{A})^{\T} \bm{A}) = K$, thus $\rank(\bm{M}^{\star }) = K$. 
As proved in \cite{huang2016anchor}, $\bm M^\natural$ satisfying SSC means ${\rm rank}(\bm M^\natural)=K$, which also means $\rank((\bm{M}^{\natural })^{\T}  (\bm{A})^{\T} \bm{A}) = K$ if ${\rm rank}(\A)=K$.
It implies that there exists an invertible matrix $\bm{Q}$ such that $ \bm{M}^{\star } = \bm{Q}^{-1} \bm{M}^{\natural}$. 
To proceed, we use the following proposition which is a part of the proof of Theorem 1 in \cite{fu2015blind}.
\begin{proposition}
    \label{proposition:det_inequality}
    For $K\leq N$, let $\bm{M} \in \mathbb{R}^{K \times N}$ be a matrix of rank $K$ and $\bm{M} \geq 0, \bm{1}^{\T}\bm{M}=\bm{1}$. For any invertible matrix $\bm{Q}$ such that $\widehat{\bm{M}} \triangleq \bm{Q}^{-1} \bm{M}$ satisfies $\widehat{\bm{M}} \geq 0, \bm{1}^{\T} \widehat{\bm{M}} = \bm{1}^{\T}$, then $\abs{\det \bm{Q}} \geq 1$. In addition, the equality holds if and only if $\bm{Q}$ is a permutation matrix.
\end{proposition}
By Proposition~\ref{proposition:det_inequality}, $\abs{\det(\bm{Q})} \geq 1$, and hence
\begin{equation}
\label{eq:det_inequality}
\det(\bm{M}^{\star }(\bm{M}^{\star })^{\T}) 
= \det(\bm{Q}^{-1} \bm{M}^{\natural} (\bm{M}^{\natural})^{\T} \bm{Q}^{-\T} )
= \det (\bm{Q})^{-2} \det \left( \bm{M}^{\natural} (\bm{M}^{\natural})^{\T} \right)
\leq \det \left( \bm{M}^{\natural} (\bm{M}^{\natural})^{\T} \right).
\end{equation} 
As we take the solution of \eqref{eq:def_optimal_theta_B} with the $\M^\star$ that has the largest volume, the equality in \eqref{eq:det_inequality} is attained. Therefore, $\bm{Q}$ can only be a permutation matrix.
The rest follows by applying Lemma~\ref{lemma:generalization}.

\section{Additional Experiments}%
\label{app:ablation}

\subsection{Effect of Neural Network Complexity}%
\label{sub:varying_network_depth}

In this section, we present some additional experiments. To be specific, we repeat experiments in Tables~\ref{table:stl10} and \ref{table:cifar10} using a more complex neural network for all the DCC methods.
The purpose is to observe the effect of $R_{\sf NET}$, i.e., the complexity of the neural network.
Here, all the DCC methods use a three-hidden-layer fully connected neural networks, where each hidden layer has 512 ReLU activation functions.
In the main text, the DCC methods used a two-hidden layer network, also with 512 activation functions in each layer.

Tables~\ref{table:ab_stl10} and \ref{table:ab_cifar10} show the new results.  Overall, similar results as in the main text can be seen: \texttt{VolMaxDCC} outperforms all other baselines, especially when noise level is getting larger.
In addition, \texttt{VolMaxDCC} enjoys a better clustering performance when using a more complex (and thus more expressive) neural network on STL10. When, the noise level is 15\%, the new ACC of \texttt{VolMaxDCC} is (0.85,0.86), which is much higher than the previous case that used a two-hidden layer neural network (0.79,0.81).
On CIFAR-10, some slight decrease of clustering accuracy occasionally appears. This may also suggest that using a deeper neural network may not always be the best choice. As implied in our theorems, a more complex neural network increases $R_{\texttt{NET}}$, which may increase the risk of overfitting---especially when $N$ is not very large. In practice, the neural network structure may be selected following standard procedures, e.g., using validation sets, if available.

\begin{table}[H]
    \centering
    \caption{Clustering performance of (seen data, unseen data) on STL10; $N_{\rm unseen}=2000$.}
    \label{table:ab_stl10}
    \resizebox{0.5\linewidth}{!}{\Huge
        \begin{tabular}{|c|c||c|c|c|c|}
            \hline
            \multicolumn{2}{|c||}{\diagbox[width=10em]{Noise \\ level}{Methods}} & DC-GMM & C-IDEC & \texttt{VanillaDCC} & \texttt{VolMaxDCC} \\ 
            \hline \hline
            \multirow{3}{*}{0.0\%} 
           & ACC & {\blue 0.89}, {\blue 0.87} &0.88, {\blue 0.87} &0.88, 0.86 &\textbf{0.92}, \textbf{0.91}  \\ 
           & NMI & {\blue 0.83}, {\blue 0.80} &0.79, 0.78 &0.79, 0.76 &\textbf{0.84}, \textbf{0.82}  \\ 
           & ARI & 0.80, 0.78 &0.77, 0.76 &{\blue 0.82}, {\blue 0.80} &\textbf{0.84}, \textbf{0.83}\\ \cline{1-6}
           \multirow{3}{*}{8.3\%} & 
             ACC &{\blue 0.78}, {\blue 0.79} &0.77, {\blue 0.79} &0.76, 0.77 & \textbf{0.85}, \textbf{0.86} \\ 
           & NMI & {\blue 0.68}, {\blue 0.70}&0.67, 0.69 &0.56, 0.59 & \textbf{0.73}, \textbf{0.75}  \\ 
           & ARI & 0.59, 0.61&0.59, 0.61 &{\blue 0.66}, {\blue 0.68} & \textbf{0.77}, \textbf{0.78} \\ \cline{1-6}
           \multirow{3}{*}{10.3\%} & 
             ACC &{\blue 0.72}, {\blue 0.74} &0.69, 0.70 &0.69, 0.71 & \textbf{0.84}, \textbf{0.85} \\ 
           & NMI &{\blue 0.63}, {\blue 0.65} &0.58, 0.60 &0.47, 0.49 & \textbf{0.72}, \textbf{0.73} \\ 
           & ARI &0.51, 0.53 &0.49, 0.51 &{\blue 0.60}, {\blue 0.61} & \textbf{0.77}, \textbf{0.78} \\ \cline{1-6}
           \multirow{3}{*}{15.0\%} & 
             ACC & {\blue 0.59}, {\blue 0.59}&0.56, 0.57 & 0.54, 0.55& \textbf{0.85}, \textbf{0.86} \\ 
           & NMI & {\blue 0.52}, {\blue 0.53}&0.49, 0.50 & 0.33, 0.34& \textbf{0.72}, \textbf{0.74}  \\ 
           & ARI & 0.36, 0.37&0.37, 0.38 & {\blue 0.48}, {\blue 0.49}& \textbf{0.77}, \textbf{0.78} \\ \hline
        \end{tabular}
    }
\end{table}
\begin{table}[H]
    \centering
    \caption{Clustering performance of (seen data, unseen data) on CIFAR10; $N_{\rm unseen}=45000$.}
    \label{table:ab_cifar10}
    \resizebox{0.5\linewidth}{!}{\Huge
        \begin{tabular}{|c|c||c|c|c|c|}
            \hline
            \multicolumn{2}{|c||}{\diagbox[width=10em]{Noise \\ level}{Methods}} & DC-GMM & C-IDEC & \texttt{VanillaDCC} & \texttt{VolMaxDCC} \\ 
            \hline \hline
            \multirow{3}{*}{0.0\%} 
           & ACC &{\blue 0.90}, {\blue 0.89} &0.88, 0.87 &0.89, 0.87 & \textbf{0.91}, \textbf{0.90} \\ 
           & NMI &\textbf{0.83}, \textbf{0.80} &{\blue 0.81}, {\blue 0.79} &0.80, 0.77 & \textbf{0.83}, \textbf{0.80} \\ 
           & ARI &{\blue 0.82}, {\blue 0.79} &0.79, 0.76 &{\blue 0.82}, {\blue 0.79} & \textbf{0.84}, \textbf{0.82}\\ \cline{1-6}
           \multirow{3}{*}{4.9\%} & 
             ACC &\textbf{0.86}, \textbf{0.86} &{\blue 0.85}, \textbf{0.86} &{\blue 0.85}, \textbf{0.86} & \textbf{0.86}, \textbf{0.86} \\ 
           & NMI &\textbf{0.77}, \textbf{0.77} &\textbf{0.77}, \textbf{0.77} &{\blue 0.73}, 0.73 & {\blue 0.73}, {\blue 0.74} \\ 
           & ARI &{\blue 0.73}, 0.73 & {\blue 0.73}, {\blue 0.74} & \textbf{0.77}, \textbf{0.77} & \textbf{0.77}, \textbf{0.77} \\ \cline{1-6}
           \multirow{3}{*}{8.7\%} & 
             ACC &{\blue 0.78}, {\blue 0.78} &0.76, 0.77 &0.76, 0.77 & \textbf{0.81}, \textbf{0.81} \\ 
           & NMI &\textbf{0.72}, \textbf{0.72} & {\blue 0.70}, {\blue 0.70} &0.58, 0.58 & 0.65, 0.66 \\ 
           & ARI &0.59, 0.60 &0.59, 0.59 &{\blue 0.70}, {\blue 0.70} & \textbf{0.73}, \textbf{0.73} \\ \cline{1-6}
           \multirow{3}{*}{10.9\%} & 
             ACC &0.70, 0.71 &{\blue 0.74}, {\blue 0.75} &0.67, 0.68 & \textbf{0.81}, \textbf{0.81} \\ 
           & NMI &0.64, {\blue 0.65} &\textbf{0.66}, \textbf{0.66} &0.47, 0.48 & {\blue 0.65}, \textbf{0.66} \\ 
           & ARI &0.50, 0.51 &0.56, 0.57 & {\blue 0.61}, {\blue 0.61} & \textbf{0.74}, \textbf{0.74}\\ \hline
        \end{tabular}
    }
\end{table}

\subsection{Effect of Pretrained Features}%
\label{sub:abc}

To see the effect of pre-training, we re-run the methods with feature vectors output by less well-trained feature extractors by \cite{li2021contrastive}. To be specific, here, the extractor is trained using 600 epochs, instead of 1000 epochs as in the main text. Nonetheless, similar results are observed in Table~\ref{table:600}. 
\begin{table}[H]
    \centering
    \caption{Clustering performance of (seen data, unseen data) on ImageNet10; $N_{\text{unseen}}=2000$; embedding vectors are extracted by training \cite{li2021contrastive} for 600 epochs.}
    \label{table:600}
    \resizebox{\linewidth}{!}{\Huge
        \begin{tabular}{|c|c||c|c|c|c|c|c|c|}
            \hline
            \multicolumn{2}{|c||}{\diagbox[width=10em]{Noise \\ level}{Methods}} & K-means & COP-Kmeans & PCKmeans & DC-GMM & C-IDEC & \texttt{VanillaDCC} & \texttt{VolMaxDCC} \\ 
            \hline \hline
           \multirow{3}{*}{0.0\%} & 
           ACC &{\blue 0.83}, ---&0.79$\pm$0.06, ---&0.74$\pm$0.07, ---&\textbf{0.97$\pm$0.00}, \textbf{0.96$\pm$0.00}&0.96$\pm$0.00, \textbf{0.96$\pm$0.00}&0.96$\pm$0.00, \textbf{0.96$\pm$0.00}& \textbf{0.97$\pm$0.00}, \textbf{0.96$\pm$0.00}  \\ 
         & NMI &0.78, ---&0.76$\pm$0.03, ---&0.76$\pm$0.03, ---&{\blue 0.92$\pm$0.00}, {\blue 0.90$\pm$0.01}&0.91$\pm$0.00, \textbf{0.91$\pm$0.00}& \textbf{0.93$\pm$0.00}, \textbf{0.91$\pm$0.00}  & \textbf{0.93$\pm$0.00}, \textbf{0.91$\pm$0.00} \\ 
         & ARI &0.66, ---&0.66$\pm$0.06, ---&0.59$\pm$0.10, ---&\textbf{0.93$\pm$0.00}, {\blue 0.90$\pm$0.01}&{\blue 0.92$\pm$0.00}, \textbf{0.91$\pm$0.00}&{\blue 0.92$\pm$0.00}, \textbf{0.91$\pm$0.00}   & {\blue 0.92$\pm$0.01}, \textbf{0.91$\pm$0.00} \\ \cline{1-9}
           \multirow{3}{*}{4.4\%} & 
           ACC &--- &0.81$\pm$0.03, ---&0.66$\pm$0.05, ---&0.91$\pm$0.00, 0.91$\pm$0.00&\textbf{0.94$\pm$0.01}, \textbf{0.94$\pm$0.01}&0.89$\pm$0.01, 0.90$\pm$0.01& {\blue 0.92$\pm$0.02}, {\blue 0.92$\pm$0.02}  \\ 
         & NMI &--- &0.75$\pm$0.01, ---&0.73$\pm$0.02, ---&{\blue 0.85$\pm$0.01}, {\blue 0.86$\pm$0.01}&\textbf{0.89$\pm$0.01}, \textbf{0.89$\pm$0.01}&0.80$\pm$0.00, 0.80$\pm$0.01& {\blue 0.85$\pm$0.02}, 0.85$\pm$0.02 \\ 
         & ARI &--- &0.66$\pm$0.03, ---&0.54$\pm$0.07, ---&0.82$\pm$0.01, 0.82$\pm$0.01&\textbf{0.88$\pm$0.02}, \textbf{0.87$\pm$0.02}&0.83$\pm$0.01, {\blue 0.84$\pm$0.01}& {\blue 0.87$\pm$0.01}, \textbf{0.87$\pm$0.01} \\ \cline{1-9}
           \multirow{3}{*}{5.8\%} & 
           ACC & ---&0.80$\pm$0.03, ---&0.73$\pm$0.07, ---&0.85$\pm$0.03, 0.85$\pm$0.04&\textbf{0.92$\pm$0.01}, \textbf{0.92$\pm$0.01}&0.81$\pm$0.00, 0.81$\pm$0.00& \textbf{0.92$\pm$0.01}, \textbf{0.92$\pm$0.01}  \\ 
         & NMI & ---&0.74$\pm$0.01, ---&0.76$\pm$0.04, ---&0.81$\pm$0.01, 0.82$\pm$0.01&\textbf{0.86$\pm$0.01}, \textbf{0.87$\pm$0.00}&0.72$\pm$0.01, 0.72$\pm$0.01& {\blue 0.84$\pm$0.01}, {\blue 0.85$\pm$0.02} \\ 
         & ARI &--- &0.65$\pm$0.03, ---&0.60$\pm$0.06, ---&0.75$\pm$0.03, 0.75$\pm$0.03&{\blue 0.84$\pm$0.01}, {\blue 0.84$\pm$0.01}&0.79$\pm$0.01, 0.79$\pm$0.01& \textbf{0.86$\pm$0.01}, \textbf{0.87$\pm$0.01} \\ \cline{1-9}
        \end{tabular}
    }
\end{table}

\subsection{Validating Identifiability Claims using Synthetic Data}%
\label{sub:synthetic_setting}
We generate synthetic data with $ N=2000, K=3$. The first $1000$ data points act as the seen samples, and the rest act as unseen data.
In order to generate valid membership $\bm{M}^{\natural}$, we sample $N$ random unit vectors, followed by adding i.i.d element-wise Gaussian noise with mean of $0$ and variance of $0.1$.
The resulting matrix is then truncated to ensure that its elements are nonnegative.
Lastly, the columns of $\bm{M}^{\natural }$ are normalized using their respective $\ell _1$-norms to become valid probability mass functions.
For constructing feature vectors $\bm{x}^{\natural}$ corresponding to membership $\bm{m}^{\natural}$, we choose a simple nonlinear function 
${\bm f}$ who is defined through its inverse, i.e.,
$\bm{f}^{-1}$ defined as:
$ \bm{f}^{-1}(\bm{m}^{\natural}) \triangleq [ 2m_1^{\natural}, 3m_2^{\natural} + 1, m_1^{\natural} m_2^{\natural} + m_3^{\natural} -2]^{\T}$. 
Such construction guarantees the existence of a function mapping from feature vector $\bm{x} \in \mathcal{X}$ to its corresponding membership vector $\bm{m}$.

To acquire the pairwise annotations, we sample $M$ pairs of indices uniformly. The pair similarity labels are then drawn from the Bernoulli distribution as described in Section~\ref{sub:generative_model}. The performance of estimated membership $\widehat{\bm{M}}$ is measured using the {\it mean squared error} (MSE), which is defined as
\[
\text{MSE}(\widehat{\bm{M}}, \bm{M}^{\natural}) 
= \dfrac{1}{K} \min_{\boldsymbol \Pi} \sum^{K}_{k=1} \norm{ \dfrac{\boldsymbol \Pi \bm{M}^{\natural}(k, :)}{\norm{\boldsymbol \Pi \bm{M}^{\natural}(k, :)}} - \dfrac{\widehat{\bm{M}}(k, :)}{ \| \widehat{\bm{M}}(k, :) \| }}^2, \quad \boldsymbol \Pi \text{ is permutation matrix}.
\] 
Fig.~\ref{fig:synthetic_clean} shows the median of the MSEs on the training and test sets of the \texttt{VanillaDCC} over 5 random trials. 
One can see the MSE is fairly low and decreases as $M$ increases---which is consistent with our theory.

\begin{figure}[t]
    \centering
    \begin{subfigure}[t]{0.45\textwidth}
        \centering
        \includegraphics[width=\textwidth]{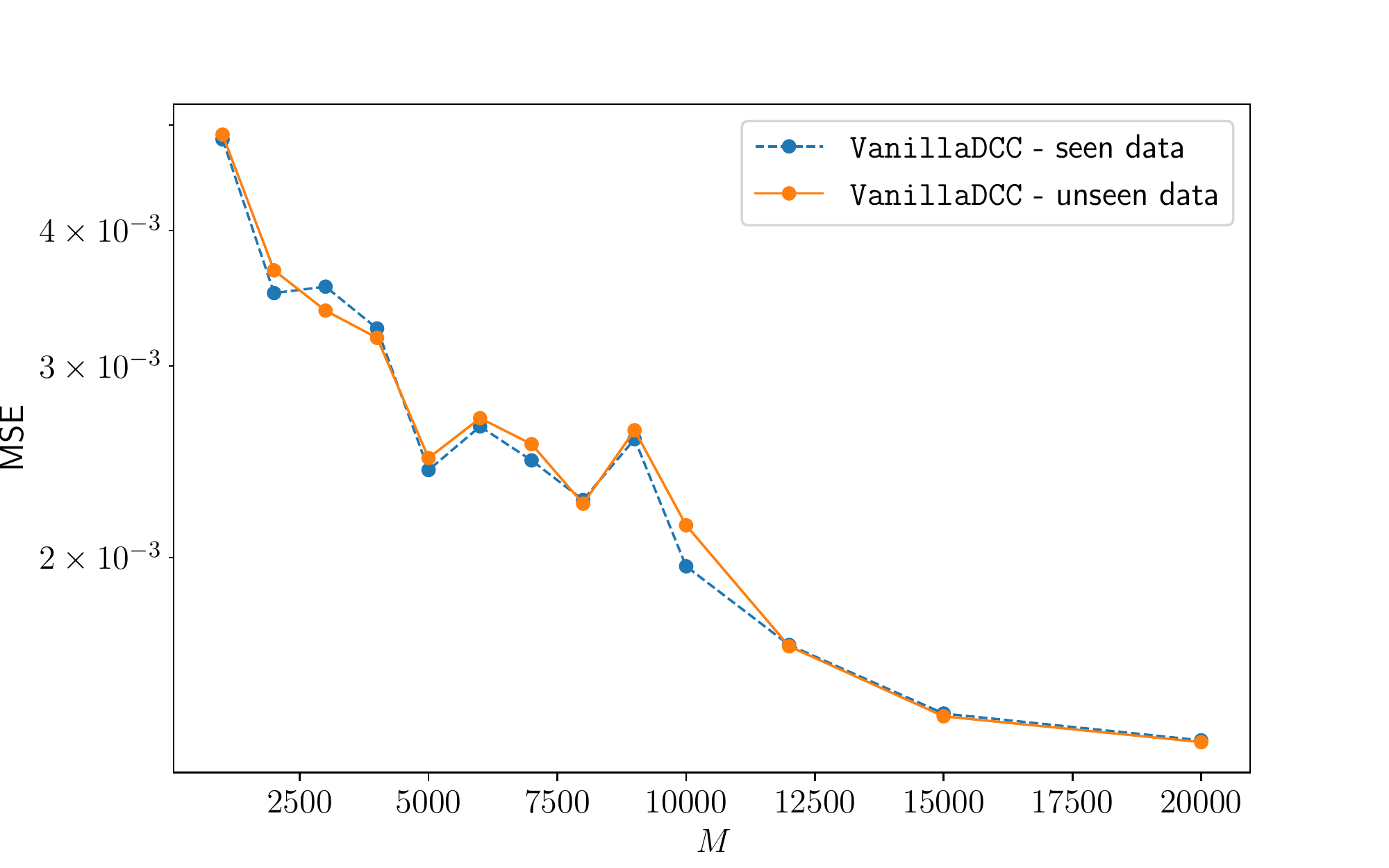}
        \caption{Average MSE of \texttt{VanillaDCC} of 5 random trials against different $M$'s.}
        \label{fig:synthetic_clean}
    \end{subfigure}
    \begin{subfigure}[t]{0.45\textwidth}
        \centering
        \includegraphics[width=\textwidth]{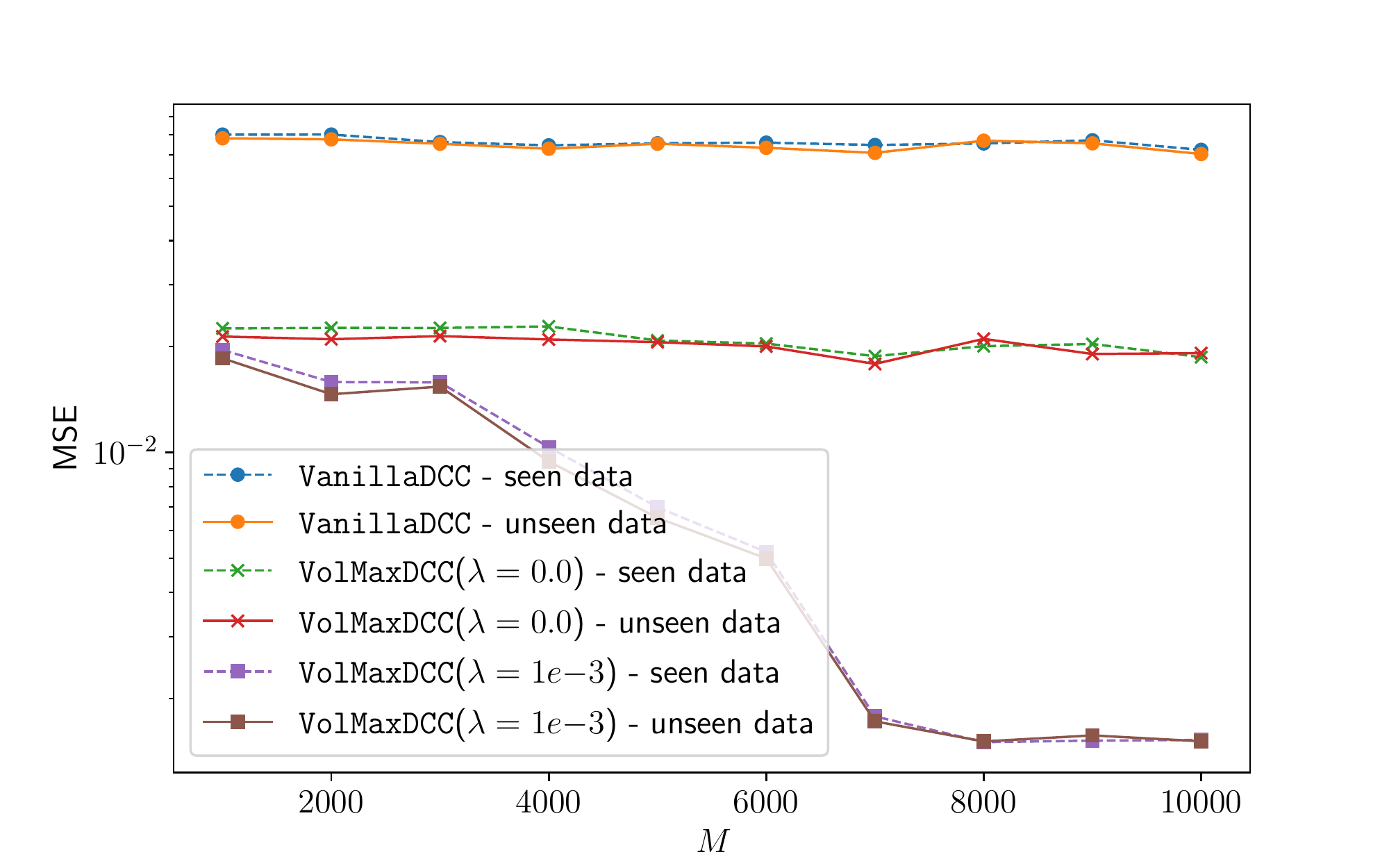}
        \caption{Median MSE of 10 random trials against different $M$'s.}
        \label{fig:synthetic_noise}
    \end{subfigure}
\end{figure}

To simulate noisy pairwise constraints, we keep the above setting and include a confusion matrix $\bm{A}$ as follows
\[
\bm{A} = \begin{bmatrix}
    1.0 & 0.2 & 0.3 \\
    0.0 & 0.8 & 0.3 \\
    0.0 & 0.0 & 0.4
\end{bmatrix}.
\] 

Fig.~\ref{fig:synthetic_noise} shows median MSE over 10 trials on seen and unseen data of \texttt{VanillaDCC} and \texttt{VolMaxDCC} using different number of pairwise constraints.
One can see that using $\lambda>0$, i.e., by promoting maximal volume of solutions, the MSE performance is much better, in the presence of noisy annotations. This shows the usefulness of the volume regularization in enhancing membership identifiability.

\subsection{Results with Standard Deviation}%
\label{sub:including_std}
In Tables~\ref{table:stl10}, \ref{table:cifar10}, \ref{table:imagenet10}, we only reported the average to avoid too-dense looking table.  We report those results here with standard deviation in Table~\ref{table:stl10_std},\ref{table:cifar10_std},\ref{table:imagenet10_std}, respectively.

\begin{table}[H]
    \centering
    \caption{Clustering performance of (seen data, unseen data) on STL10; $N_{\rm unseen}=2000$.}
    \label{table:stl10_std}
    \resizebox{\linewidth}{!}{\Huge
        \begin{tabular}{|c|c||c|c|c|c|c|c|c|}
            \hline
            \multicolumn{2}{|c||}{\diagbox[width=10em]{Noise \\ level}{Methods}} & Kmeans & COP-Kmeans & PCKmeans & DC-GMM & C-IDEC & \texttt{VanillaDCC} & \texttt{VolMaxDCC} \\ 
            \hline \hline
            \multirow{3}{*}{0.0\%} 
           & ACC & 0.71$\pm${0.03}, --- & 0.66$\pm$0.04 , --- & 0.70$\pm$0.08, --- & 0.88 $\pm$ 0.04, 0.87$\pm$ 0.04  & 0.89$\pm$0.01, 0.88$\pm$0.02 & \textbf{0.93$\pm$0.00}, \textbf{0.91$\pm$0.00} & {\blue 0.91$\pm$0.03}, {\blue 0.89$\pm$0.03} \\
           & NMI & 0.75$\pm${0.01} , --- & 0.67$\pm$0.02 , --- & 0.72$\pm$0.02, --- & {\blue 0.82$\pm$0.02}, {\blue 0.80$\pm$0.02} & 0.81$\pm$0.01, {\blue 0.80$\pm$0.02} & \textbf{0.84$\pm$0.01}, \textbf{0.81$\pm$0.01} & {\blue 0.82$\pm$0.03}, {\blue 0.80$\pm$0.03} \\
           & ARI & 0.54$\pm$0.03 , --- & 0.52$\pm$0.03 , --- & 0.55$\pm$0.06, --- & 0.80$\pm$0.04, 0.78$\pm$0.03  & 0.79$\pm$0.02, 0.77$\pm$0.02 & \textbf{0.85$\pm$0.00}, \textbf{0.83$\pm$0.00} & {\blue 0.84$\pm$0.01}, {\blue 0.82$\pm$0.01} \\ \cline{1-9}
           \multirow{3}{*}{8.3\%} & 
             ACC & --- & 0.70$\pm$0.04, --- & 0.70$\pm$0.04, --- & 0.75$\pm$0.03, 0.76$\pm$0.03 & 0.77$\pm$0.01, 0.79$\pm$0.01 & {\blue 0.78$\pm$0.00}, {\blue 0.80$\pm$0.01} & \textbf{0.80$\pm$0.01}, \textbf{0.81$\pm$0.00} \\
           & NMI & --- & 0.64$\pm$0.03, --- & \textbf{0.71}$\pm$0.02, --- & {\blue 0.67$\pm$0.01}, \textbf{0.69$\pm$0.01} & {\blue 0.67$\pm$0.01}, \textbf{0.69$\pm$0.01} & 0.59$\pm$0.00, 0.62$\pm$0.01 &  0.64$\pm$0.02, {\blue 0.65$\pm$0.02} \\ 
           & ARI & --- & 0.51$\pm$0.03, --- & 0.56$\pm$0.04, --- & 0.57$\pm$0.02, 0.59$\pm$0.02 & 0.59$\pm$0.02, 0.61$\pm$0.01 & {\blue 0.69$\pm$0.00}, {\blue 0.71$\pm$0.01} & \textbf{0.73$\pm$0.02}, \textbf{0.74$\pm$0.02} \\ \cline{1-9}
           \multirow{3}{*}{10.3\%} & 
             ACC & --- & 0.62$\pm$0.05, --- & 0.69$\pm$0.04, --- & 0.70$\pm$0.02, 0.72$\pm$0.03 & 0.70$\pm$0.03, 0.71$\pm$0.03 & {\blue 0.72$\pm$0.01}, {\blue 0.73$\pm$0.01} & \textbf{0.79$\pm$0.00}, \textbf{0.81$\pm$0.00} \\ 
           & NMI & --- & 0.59$\pm$0.02, --- & \textbf{0.73}$\pm$0.01, --- & 0.62$\pm$0.02, {\blue 0.64$\pm$0.02} &0.60$\pm$0.02, 0.62$\pm$0.02 & 0.50$\pm$0.01, 0.51$\pm$0.01 & {\blue 0.68$\pm$0.00}, \textbf{0.70$\pm$0.00} \\ 
           & ARI & --- & 0.44$\pm$0.03, --- & 0.55$\pm$0.02, --- & 0.51$\pm$0.01, 0.52$\pm$0.02 & 0.50$\pm$0.03, 0.52$\pm$0.03 & {\blue 0.62$\pm$0.00}, {\blue 0.64$\pm$0.01} & \textbf{0.77$\pm$0.01}, \textbf{0.78$\pm$0.00} \\ \cline{1-9}
           \multirow{3}{*}{15.0\%} & 
             ACC & --- & 0.62$\pm$0.05, --- & {\blue 0.64}$\pm$0.07, --- & 0.60$\pm$0.02, {\blue 0.61$\pm$0.02} & 0.57$\pm$0.01, 0.57$\pm$0.01 & 0.56$\pm$0.03, 0.58$\pm$0.03 & \textbf{0.79$\pm$0.00}, \textbf{0.81$\pm$0.00} \\ 
           & NMI & --- & 0.54$\pm$0.01, --- & \textbf{0.72}$\pm$0.01, --- & 0.54$\pm$0.01, {\blue 0.55$\pm$0.02} & 0.50$\pm$0.02, 0.50$\pm$0.02 & 0.33$\pm$0.01, 0.35$\pm$0.01 & {\blue 0.68$\pm$0.01}, \textbf{0.69$\pm$0.01} \\ 
           & ARI & --- & 0.41$\pm$0.02, --- & {\blue 0.52}$\pm$0.05, --- & 0.38$\pm$0.02, 0.39$\pm$0.01 & 0.38$\pm$0.01, 0.39$\pm$0.02 & 0.50$\pm$0.01, {\blue 0.51$\pm$0.01} & \textbf{0.76$\pm$0.02}, \textbf{0.77$\pm$0.02} \\ \hline
        \end{tabular}
    }
\end{table}
\begin{table}[H]
    \centering
    \caption{Clustering performance of (seen data, unseen data) on CIFAR10; $N_{\rm unseen}=45000$.}
    \label{table:cifar10_std}
    \resizebox{\linewidth}{!}{\Huge
        \begin{tabular}{|c|c||c|c|c|c|c|c|c|}
            \hline
            \multicolumn{2}{|c||}{\diagbox[width=10em]{Noise \\ level}{Methods}} & Kmeans & COP-Kmeans & PCKmeans & DC-GMM & C-IDEC & \texttt{VanillaDCC} & \texttt{VolMaxDCC} \\ 
            \hline \hline
            \multirow{3}{*}{0.0\%} 
           & ACC & 0.78$\pm$0.00 , --- & 0.67$\pm$0.06, --- & 0.67$\pm$0.06, --- & {\blue 0.91$\pm$0.01}, {\blue 0.89$\pm$0.01} &0.90$\pm$0.01, {\blue 0.89$\pm$0.01} &\textbf{0.92$\pm$0.00}, \textbf{0.90$\pm$0.00} & {\blue 0.91$\pm$0.01}, \textbf{0.90$\pm$0.00}\\
           & NMI & 0.71$\pm$0.00 , --- & 0.66$\pm$0.02, --- & 0.71$\pm$0.02, --- & {\blue 0.83$\pm$0.01}, \textbf{0.81$\pm$0.01}& {\blue 0.83$\pm$0.01}, \textbf{0.81$\pm$0.01}&\textbf{0.84$\pm$0.00}, {\blue 0.80$\pm$0.00} & {\blue 0.83$\pm$0.01}, {\blue 0.80$\pm$0.01} \\
           & ARI & 0.62$\pm$0.00 , --- & 0.54$\pm$0.05, --- & 0.55$\pm$0.05, --- &0.82$\pm$0.01, 0.79$\pm$0.01 & 0.81$\pm$0.01, 0.79$\pm$0.01& \textbf{0.85$\pm$0.00}, {\blue 0.81$\pm$0.00}& {\blue 0.84$\pm$0.00}, \textbf{0.82$\pm$0.00} \\ \cline{1-9}
           \multirow{3}{*}{4.9\%} & 
             ACC & --- & {\blue 0.75}$\pm$0.04, --- & 0.70$\pm$0.06, --- &\textbf{0.86$\pm$0.00}, \textbf{0.86$\pm$0.00} &\textbf{0.86$\pm$0.00}, \textbf{0.86$\pm$0.00} & \textbf{0.86$\pm$0.00}, \textbf{0.86$\pm$0.00} & \textbf{0.86$\pm$0.00}, \textbf{0.86$\pm$0.00} \\ 
           & NMI & --- & 0.69$\pm$0.01, --- & 0.69$\pm$0.02, --- &\textbf{0.77$\pm$0.00}, \textbf{0.77$\pm$0.00} & \textbf{0.77$\pm$0.00}, \textbf{0.77$\pm$0.00}& {\blue 0.73$\pm$0.00}, 0.73$\pm$0.00 & {\blue 0.73$\pm$0.00}, {\blue 0.74$\pm$0.00}\\ 
           & ARI & --- & 0.60$\pm$0.03, --- & 0.57$\pm$0.04, --- &{\blue 0.73$\pm$0.00}, {\blue 0.74$\pm$0.00} &{\blue 0.73$\pm$0.00}, {\blue 0.74$\pm$0.00} &\textbf{0.77$\pm$0.00}, \textbf{0.77$\pm$0.00} & \textbf{0.77$\pm$0.00}, \textbf{0.77$\pm$0.00} \\ \cline{1-9}
           \multirow{3}{*}{8.7\%} & 
             ACC & --- & 0.64$\pm$0.06, ---  & 0.72$\pm$0.06, --- &{\blue 0.76$\pm$0.04}, 0.76$\pm$0.04 & {\blue 0.76$\pm$0.02}, 0.76$\pm$0.02&{\blue 0.76$\pm$0.02}, {\blue 0.77$\pm$0.02} & \textbf{0.83$\pm$0.01}, \textbf{0.83$\pm$0.01} \\ 
           & NMI & --- & 0.63$\pm$0.02, --- & 0.69$\pm$0.02, --- &\textbf{0.71$\pm$0.02}, \textbf{0.71$\pm$0.02} &{\blue 0.70$\pm$0.01}, {\blue 0.70$\pm$0.01} & 0.58$\pm$0.01, 0.59$\pm$0.01& {\blue 0.70$\pm$0.01}, {\blue 0.70$\pm$0.01} \\ 
           & ARI & --- & 0.49$\pm$0.04, --- & 0.59$\pm$0.03, --- &0.58$\pm$0.03, 0.59$\pm$0.03 &0.59$\pm$0.01, 0.60$\pm$0.01 & {\blue 0.70$\pm$0.01}, {\blue 0.70$\pm$0.01}& \textbf{0.75$\pm$0.01}, \textbf{0.75$\pm$0.01}\\ \cline{1-9}
           \multirow{3}{*}{10.9\%} & 
             ACC & --- & 0.68$\pm$0.06, --- & 0.70$\pm$0.04, --- &{\blue 0.74$\pm$0.03}, {\blue 0.74$\pm$0.02} &0.73$\pm$0.02, 0.73$\pm$0.02 &0.68$\pm$0.00, 0.69$\pm$0.01  & \textbf{0.82$\pm$0.01}, \textbf{0.82$\pm$0.01} \\ 
           & NMI & --- & 0.61$\pm$0.03, --- & \textbf{0.68}$\pm$0.01, --- &{\blue 0.67$\pm$0.02}, \textbf{0.68$\pm$0.02} &0.65$\pm$0.01, {\blue 0.66$\pm$0.01} &0.48$\pm$0.01, 0.49$\pm$0.01 &  \textbf{0.68$\pm$0.02}, \textbf{0.68$\pm$0.02}\\ 
           & ARI & --- & 0.50$\pm$0.05, --- & 0.57$\pm$0.03, --- &0.55$\pm$0.03, 0.55$\pm$0.03 &0.56$\pm$0.03, 0.57$\pm$0.03 & {\blue 0.62$\pm$0.01}, {\blue 0.63$\pm$0.01}& \textbf{0.74$\pm$0.01}, \textbf{0.74$\pm$0.01} \\ \hline
        \end{tabular}
    }
\end{table}
\begin{table}[H]
    \centering
    \caption{Clustering performance of (seen data, unseen data) on ImageNet10; $N_{\rm unseen}=2000$.}
    \label{table:imagenet10_std}
    \resizebox{\linewidth}{!}{\Huge
        \begin{tabular}{|c|c||c|c|c|c|c|c|c|}
            \hline
            \multicolumn{2}{|c||}{\diagbox[width=10em]{Noise \\ level}{Methods}} & Kmeans & COP-Kmeans & PCKmeans & DC-GMM & C-IDEC & \texttt{VanillaDCC} & \texttt{VolMaxDCC} \\ 
            \hline \hline
            \multirow{3}{*}{0.0\%} 
           & ACC & {\blue 0.85$\pm$0.00 }, --- & 0.79$\pm$0.06, --- & 0.70$\pm$0.06, --- & \textbf{0.97$\pm$0.00}, \textbf{0.96$\pm$0.00} &\textbf{0.97$\pm$0.00}, \textbf{0.96$\pm$0.00} &\textbf{0.97$\pm$0.00}, \textbf{0.96$\pm$0.00} & \textbf{0.97$\pm$0.00}, \textbf{0.96$\pm$0.00} \\
           & NMI & 0.80$\pm$0.00, --- & 0.77$\pm$0.03, --- & 0.75$\pm$0.04, --- &{\blue 0.93$\pm$0.00}, {\blue 0.91$\pm$0.00} &{\blue 0.93$\pm$0.00}, \textbf{0.92$\pm$0.00} & \textbf{0.94$\pm$0.00}, {\blue 0.91$\pm$0.00} & \textbf{0.94$\pm$0.00}, {\blue 0.91$\pm$0.00} \\ 
           & ARI & 0.68$\pm$0.00, --- & 0.66$\pm$0.05, --- & 0.55$\pm$0.07, --- &\textbf{0.94$\pm$0.00}, \textbf{0.92$\pm$0.00} &\textbf{0.94$\pm$0.00}, \textbf{0.92$\pm$0.00} & {\blue 0.93$\pm$0.00}, {\blue 0.91$\pm$0.00} & {\blue 0.93$\pm$0.00}, {\blue 0.91$\pm$0.00} \\ \cline{1-9}
           \multirow{3}{*}{3.4\%} & 
             ACC & --- & 0.79$\pm$0.09, --- & 0.66$\pm$0.12, --- &{\blue 0.93$\pm$0.01}, 0.92$\pm$0.01 &{\blue 0.93$\pm$0.00}, {\blue 0.93$\pm$0.00} &0.92$\pm$0.01, 0.91$\pm$0.01 & \textbf{0.94$\pm$0.01}, \textbf{0.94$\pm$0.01} \\ 
           & NMI & --- & 0.75$\pm$0.04, --- & 0.73$\pm$0.05, --- & 0.86$\pm$0.01, 0.85$\pm$0.01 & {\blue 0.87$\pm$0.01}, {\blue 0.86$\pm$0.01}&0.83$\pm$0.02, 0.82$\pm$0.02 & \textbf{0.88$\pm$0.01}, \textbf{0.87$\pm$0.02} \\ 
           & ARI & --- & 0.65$\pm$0.07, --- & 0.52$\pm$0.12, --- &0.84$\pm$0.01, 0.83$\pm$0.01 &{\blue 0.86$\pm$0.01}, {\blue 0.85$\pm$0.01} &0.84$\pm$0.01, 0.84$\pm$0.01 & \textbf{0.89$\pm$0.01}, \textbf{0.88$\pm$0.01} \\ \cline{1-9}
           \multirow{3}{*}{6.9\%} & 
             ACC & --- & 0.74$\pm$0.07, --- & 0.72$\pm$0.08, --- &0.84$\pm$0.01, 0.84$\pm$0.01 &{\blue 0.88$\pm$0.01}, {\blue 0.88$\pm$0.01} & 0.84$\pm$0.00, 0.84$\pm$0.00 & \textbf{0.92$\pm$0.00}, \textbf{0.91$\pm$0.00} \\ 
           & NMI & --- & 0.70$\pm$0.03, --- & 0.76$\pm$0.04, --- &0.79$\pm$0.01, 0.79$\pm$0.00 &{\blue 0.82$\pm$0.01}, {\blue 0.82$\pm$0.01} & 0.70$\pm$0.01, 0.70$\pm$0.00 & \textbf{0.84$\pm$0.00}, \textbf{0.83$\pm$0.00} \\ 
           & ARI & --- & 0.58$\pm$0.05, --- & 0.59$\pm$0.09, --- & 0.71$\pm$0.01, 0.71$\pm$0.01&{\blue 0.77$\pm$0.02}, {\blue 0.77$\pm$0.02} & {\blue 0.77$\pm$0.01}, {\blue 0.77$\pm$0.01} & \textbf{0.88$\pm$0.00}, \textbf{0.87$\pm$0.00} \\ \cline{1-9}
           \multirow{3}{*}{11.2\%} & 
             ACC & --- & 0.72$\pm$0.05, --- & 0.62$\pm$0.07, --- &0.71$\pm$0.03, 0.72$\pm$0.03 &{\blue 0.80$\pm$0.02}, {\blue 0.81$\pm$0.02} &0.65$\pm$0.01, 0.66$\pm$0.01 & \textbf{0.91$\pm$0.01}, \textbf{0.90$\pm$0.01} \\ 
           & NMI & --- & 0.64$\pm$0.02, --- & 0.73$\pm$0.03, --- &0.68$\pm$0.01, 0.70$\pm$0.01 & {\blue 0.74$\pm$0.02}, {\blue 0.76$\pm$0.02}&0.49$\pm$0.03, 0.52$\pm$0.02 & \textbf{0.83$\pm$0.01}, \textbf{0.82$\pm$0.01} \\ 
           & ARI & --- & 0.54$\pm$0.03, --- & 0.51$\pm$0.09, --- &0.56$\pm$0.02, 0.58$\pm$0.02 &{\blue 0.66$\pm$0.03}, {\blue 0.68$\pm$0.03} & 0.62$\pm$0.03, 0.63$\pm$0.02 & \textbf{0.87$\pm$0.01}, \textbf{0.86$\pm$0.00} \\ \hline
        \end{tabular}
    }
\end{table}

\end{document}